\documentclass{article}

\usepackage{graphicx,amsmath,amsfonts,amssymb,bm,url,breakurl,epsfig,epsf,color,fullpage,MnSymbol,mathbbol,fmtcount,semtrans,cite,caption,subcaption,multirow,comment}
\usepackage{algpseudocode}

\usepackage[utf8]{inputenc} 
\usepackage[T1]{fontenc}    
\usepackage{url}            
\usepackage{booktabs}       
\usepackage{amsfonts}       
\usepackage{nicefrac}       
\usepackage{microtype}      

  \makeatletter
\def\BState{\State\hskip-\ALG@thistlm}
\makeatother

  \usepackage{mathtools}

\usepackage{titlesec}

\usepackage{tikz}
\usepackage{pgfplots}
\usetikzlibrary{pgfplots.groupplots}

\titleformat{\paragraph}
{\normalfont\normalsize\bfseries}{\theparagraph}{1em}{}
\titlespacing*{\paragraph}
{0pt}{3.25ex plus 1ex minus .2ex}{1.5ex plus .2ex}

\usepackage{movie15}

\usepackage{cite}
\usepackage{caption}
\usepackage[bottom,hang,flushmargin]{footmisc} 

\usepackage{hyperref}
\definecolor{darkred}{RGB}{150,0,0}
\definecolor{darkgreen}{RGB}{0,150,0}
\definecolor{darkblue}{RGB}{0,0,200}
\hypersetup{colorlinks=true, linkcolor=darkred, citecolor=darkgreen, urlcolor=darkblue}

\newtheorem{theorem}{Theorem}[section]

\newtheorem{lemma}[theorem]{Lemma}
\newtheorem{corollary}[theorem]{Corollary}

\newtheorem{definition}[theorem]{Definition}

\newcommand{\eps}{\varepsilon}

\newcommand{\beq}{\begin{equation}}
\newcommand{\eeq}{\end{equation}}

\newcommand{\nn}{\nonumber}

\newcommand{\A}{{\mtx{A}}}

\newcommand{\Ub}{{\mtx{U}}}
\newcommand{\bN}{{\bar{N}}}

\newcommand{\V}{{\mtx{V}}}

\newcommand{\B}{{{\mtx{B}}}}

\newcommand{\Sb}{{{\mtx{S}}}}
\newcommand{\Gb}{{\mtx{G}}}

\newcommand{\Ac}{{\cal{A}}}

\newcommand{\TK}{{T}}
\newcommand{\NW}{{N_w}}

\newcommand{\Ginf}{\boldsymbol{\Gamma}_\infty}
\newcommand{\Binf}{\boldsymbol{\bar{\Gamma}}_\infty}
\newcommand{\cba}[1]{\Gb_{#1}}

\newcommand{\cbe}[1]{\hat{\Gb}_{#1}}
\newcommand{\ca}[1]{\F_{#1}}

\newcommand{\Rbo}[1]{{\mtx{R}_{#1}}}

\newcommand{\Cb}{{\mtx{C}}}
\newcommand{\Cbb}{{\mtx{\bar{C}}}}

\newcommand{\Ab}{{\mtx{\bar{A}}}}
\newcommand{\HK}{{\text{Ho-Kalman}}}
\newcommand{\Ah}{{\mtx{\hat{A}}}}
\newcommand{\Bh}{{\mtx{\hat{B}}}}
\newcommand{\Ch}{{\mtx{\hat{C}}}}
\newcommand{\Bb}{{\mtx{\bar{B}}}}
\newcommand{\Tb}{{\mtx{T}}}
\newcommand{\Hb}{{\mtx{H}}}

\newcommand{\Hbp}{\mtx{H}^+}  

\newcommand{\Hbm}{\mtx{H}^-}  

\newcommand{\Hbh}{\mtx{{\hat{L}}}} 

\newcommand{\Hbhp}{\mtx{\hat{H}}^+} 

\newcommand{\Hbhm}{\mtx{\hat{H}}^-} 
\newcommand{\Hbhl}{\mtx{{\hat{L}}}} 

\newcommand{\Hbl}{\mtx{{L}}} 

\newcommand{\smin}{\sigma_{\min}(\Hbl)}

\newcommand{\Ob}{{\mtx{O}}} 

\newcommand{\Obh}{{\mtx{\hat{O}}}}

\newcommand{\Qb}{{\mtx{Q}}} 

\newcommand{\Qbh}{{\mtx{\hat{Q}}}}

\newcommand{\F}{{\mtx{F}}}

\newcommand{\bSi}{{\boldsymbol{{\Sigma}}}}

\newcommand{\bmu}{{\boldsymbol{{\mu}}}}
\newcommand{\Db}{{\mtx{D}}}
\newcommand{\db}{{\vct{d}}}

\newcommand{\Iden}{{\mtx{I}}}
\newcommand{\M}{{\mtx{M}}}

\newcommand{\Lb}{{\mtx{L}}}

\newcommand{\order}[1]{{\cal{O}}(#1)}

\newcommand{\z}{{\vct{z}}}

\newcommand{\tn}[1]{\|{#1}\|_{\ell_2}}
\newcommand{\tf}[1]{\|{#1}\|_{F}}

\newcommand{\Cc}{\mathcal{C}}

\newcommand{\Sc}{\mathcal{S}}

\newcommand{\Nn}{\mathcal{N}}

\newcommand{\vb}{\vct{v}}

\newcommand{\w}{\vct{w}}

\newcommand{\ab}{\vct{a}}
\newcommand{\bb}{\vct{b}}
\newcommand{\ub}{{\vct{u}}}
\newcommand{\cb}{{\vct{c}}}
\newcommand{\ubb}{\bar{\vct{u}}}
\newcommand{\h}{\vct{h}}

\newcommand{\g}{{\vct{g}}}

\newcommand{\Zb}{\mtx{Z}}

\newcommand{\Fc}{\mathcal{F}}

\newcommand{\rrr}[1]{{\textcolor{black}{{#1}}}}

\newcommand{\yh}{\hat{\y}}

\newcommand{\m}{\vct{m}}

\newcommand{\wb}{\bar{\w}}

\newcommand{\x}{\vct{x}}

\newcommand{\y}{\vct{y}}

\newcommand{\W}{\mtx{W}}
\newcommand{\bgl}{{~\big |~}}

\definecolor{emmanuel}{RGB}{255,127,0}

\newcommand{\R}{\mathbb{R}}
\newcommand{\Pro}{\mathbb{P}}

\newcommand{\tr}[1]{{\text{tr}(#1)}}
\newcommand{\E}{\operatorname{\mathbb{E}}}
\newcommand{\Eb}{\mtx{E}}

\newcommand{\e}{\vct{e}}
\newcommand{\eb}{{\vct{e}}}

\newcommand{\vct}[1]{\bm{#1}}
\newcommand{\mtx}[1]{\bm{#1}}

\newcommand{\Pc}{{\cal{P}}}
\newcommand{\X}{{\mtx{X}}}
\newcommand{\Y}{{\mtx{Y}}}
\newcommand{\Vb}{{\mtx{V}}}

\numberwithin{equation}{section} 

\def \endprf{\hfill {\vrule height6pt width6pt depth0pt}\medskip}
\newenvironment{proof}{\noindent {\bf Proof} }{\endprf\par}

\usepackage[utf8]{inputenc} 
\usepackage[T1]{fontenc}    
\usepackage{url}            
\usepackage{booktabs}       
\usepackage{amsfonts}       
\usepackage{nicefrac}       
\usepackage{microtype}      

\usepackage{algorithm}

\makeatletter
\def\BState{\State\hskip-\ALG@thistlm}
\makeatother

\title{Non-asymptotic Identification of LTI Systems\\from a Single Trajectory}

\author{Samet Oymak\thanks{{Department of Electrical and Computer Engineering, University of California, Riverside, CA.}}\quad and\quad Necmiye Ozay\thanks{Electrical Engineering and Computer Science Department, University of Michigan, Ann Arbor, MI.}~\footnote{Version 2 has two improvements: First, paper now uses spectral radius rather than largest singular value hence applies to a larger class of systems. Secondly, new sample complexity bounds are provided for approximating the system's Hankel operator via estimated Markov parameters. These bounds leverage stability and treat the system as if it has a logarithmic order.}}
\date{}
\begin{document}
\maketitle

\begin{abstract} We consider the problem of learning a realization for a linear time-invariant (LTI) dynamical system from input/output data. Given a single input/output trajectory, we provide finite time analysis for learning the system's Markov parameters, from which a balanced realization is obtained using the classical Ho-Kalman algorithm. By proving a stability result for the Ho-Kalman algorithm and combining it with the sample complexity results for Markov parameters, we show how much data is needed to learn a balanced realization of the system up to a desired accuracy with high probability. 
\end{abstract}

\section{Introduction}
Many modern control design techniques rely on the existence of a fairly accurate state-space model of the plant to be controlled. Although in some cases a model can be obtained from first principles, there are many situations in which a model should be learned from input/output data. Classical results in system identification provide asymptotic convergence guarantees for learning models from data \cite{ljung1998system,van2012subspace}. However, finite sample complexity properties have been rarely discussed in system identification literature \cite{weyer1999finite}; and earlier results are conservative \cite{lwom}. 

There is recent interest from the machine learning community in data-driven control and non-asymptotic analysis. Putting aside the reinforcement learning literature and restricting our attention to linear state-space models, the work in this area can be divided into two categories: (i) directly learning the control inputs to optimize a control objective or analyzing the predictive power of the learned representation \cite{faradonbeh2017finite,hazan2017learning,fazel2018global}, (ii) learning the parameters of the system model from limited data \cite{pereira2010learning,hardt2016gradient,dean2017sample,boczar2018finite,lwom,arora2018towards}. For the former problem, the focus has been on exploration/exploitation type formulations and regret analysis. Since the goal is to learn how to control the system to achieve a specific task, the system is not necessarily fully learned. On the other hand, the latter problem aims to learn a general purpose model that can be used in different control tasks, for instance, by combining it with robust control techniques \cite{dean2017sample,tu2017non,boczar2018finite}. The focus for the latter work has been to analyze data--accuracy trade-offs. 








In this paper we focus on learning a realization for an LTI system from a single \emph{input/output} trajectory. This setting is significantly more challenging than earlier studies that assume that (multiple independent) \emph{state} trajectories are available  \cite{dean2017sample,lwom}. One of our main contributions is to derive sample complexity results in learning the Markov parameters, to be precisely defined later, of the system using a least squares algorithm \cite{fledderjohn2010comparison}. Markov parameters play a central role in system identification \cite{ljung1998system} and they can also be directly used in control design when the system model itself is not available \cite{skelton1994data,furuta1995discrete,santillo2010adaptive}. In Section \ref{sec low ord}, we show that using few Markov parameter estimates and leveraging stability assumption, one can approximate system's Hankel operator with near optimal sample size. When only input/output data is available, it is well known that the system matrices can be identified only up to a similarity transformation even in the noise-free case but Markov parameters are identifiable. Therefore, we focus on obtaining a realization. One classical technique to derive a realization from the Markov parameters is the Ho-Kalman (a.k.a., eigensystem realization algorithm -- ERA) algorithm \cite{ho1966effective}. The Ho-Kalman algorithm constructs a balanced realization\footnote{Balanced realizations give a representation of the system in a basis that orders the states in terms of their effect on the input/output behavior. This is relevant for determining the system order and for model reduction \cite{sanchez1998robust}.} for the system from the singular value decomposition of the Hankel matrix of the Markov parameters. By proving a stability result for the Ho-Kalman algorithm and combining it with the sample complexity results, we show how much data is needed to learn a balanced realization of the system up to a desired accuracy with high probability.








\section{Problem Setup}\label{sec setup}
We first introduce the basic notation. Spectral norm $\|\cdot\|$ returns the largest singular value of a matrix. Multivariate normal distribution with mean $\bmu$ and covariance matrix $\bSi$ is denoted by $\Nn(\bmu,\bSi)$. $\X^*$ denotes the transpose of a matrix $\X$. $\X^\dagger$ returns the Moore--Penrose inverse of the matrix $\X$. Covariance matrix of a random vector $\vb$ is denoted by $\bSi(\vb)$. $\tr{\cdot}$ returns the trace of a matrix. $c,C,c',c_1,c_2,\dots$ stands for absolute constants.

Suppose we have an observable and controllable linear system characterized by the system matrices $\A\in\R^{n\times n},\B\in\R^{n\times p},\Cb\in\R^{m\times n},\Db\in\R^{m\times p}$ and this system evolves according to
\begin{align}
&\x_{t+1}=\A\x_t+\B\ub_t+\w_t,\\\label{main rel}
&\y_{t}=\Cb\x_t+\Db\ub_t+\z_t.
\end{align}
Our goal is to learn the characteristics of this system and to provide finite sample bounds on the estimation accuracy. Given a horizon $\TK$, we will learn the first $\TK$ Markov parameters of the system. The first Markov parameter is the matrix $\Db$, and the remaining parameters are the set of matrices $\{\Cb\A^i\B\}_{i=0}^{\TK-2}$. As it will be discussed later on, by learning these parameters,
\begin{itemize}
\item we can provide bounds on how well $\y_t$ can be estimated for a future time $t$,
\item we can identify the state-space matrices $\A,\B,\Cb,\Db$ (up to a similarity transformation).
\end{itemize}

{\bf{Problem setup:}} We assume that $\{\ub_t,\w_t,\z_t\}_{t=1}^{\infty}$ are vectors that are independent of each other with distributions $\ub_t\sim\Nn(0,\sigma_u^2\Iden_p)$, $\w_t\sim\Nn(0,\sigma_w^2\Iden_n)$, and $\z_t\sim\Nn(0,\sigma_z^2\Iden_m)$\footnote{While we assume diagonal covariance throughout the paper, we believe our proof strategy can be adapted to arbitrary covariance matrices.}. $\ub_t$ is the input vector which is known to us. $\w_t$ and $\z_t$ are the process and measurement noise vectors respectively. We also assume that the initial condition of the hidden state is $\x_1=0$. Observe that Markov parameters can be found if we have access to cross correlations $\E[\y_t\ub_{t-k}^*]$. In particular, we have the identities
\[
\E\left[\frac{\y_t\ub_{t-k}^*}{\sigma_u^2}\right]=\begin{cases}\Db~~~\text{if}~~~k=0,\\\Cb\A^{k-1}\B~~~\text{if}~~~k\geq 1\end{cases}.
\]
Hence, if we had access to infinitely many independent $(\y_t,\ub_{t-k})$ pairs, our task could be accomplished by a simple averaging. In this work, we will show that, one can robustly learn these matrices from a small amount of data generated from a single realization of the system trajectory. The challenge is efficiently using finite and dependent data points to perform reliable estimation. Observe that, our problem is identical to learning the concatenated matrix $\cba{}$ defined as
\[
\cba{}=[\Db,~\Cb\B,~\Cb\A\B,~\dots,~\Cb\A^{\TK-2}\B]\in\R^{m\times \TK p}.
\]
Next section describes our input and output data. Based on this, we formulate a least-squares procedure that estimates $\cba{}$. The estimate $\cbe{}$ will play a critical role in the identification of the system matrices.

\subsection{Least-Squares Procedure}\label{lsp sec}
To describe the estimation procedure, we start by explaining the data collection process.
Given a single input/output trajectory $\{\y_t,\ub_t\}_{t=1}^{\bN}$, we generate $N$ subsequences of length $\TK$, where $\bN=\TK+N-1$ and $N\geq 1$. To ease representation, we organize the data $\ub_t$ and the noise $\w_t$ into length $\TK$ chunks denoted by the following vectors,
\begin{align}
&\ubb_t=[\ub_{t}^*~\ub_{t-1}^*~\dots~\ub_{t-\TK+1}^*]^*\in\R^{\TK p},\\
&\wb_t=[\w_{t}^*~\w_{t-1}^*~\dots~\w_{t-\TK+1}^*]^*\in\R^{\TK n}.
\end{align}
In a similar fashion to $\cba{}$ define the matrix, 
\begin{align*}
&\ca{}=[{\mtx{0}}~\Cb~\Cb\A~\dots~\Cb\A^{\TK-2}]\in\R^{m\times \TK n}.
\end{align*}
To establish an explicit connection to Markov parameters, $\y_t$ can be expanded recursively until $t-\TK+1$ to relate the output to the input $\ubb_t$ and Markov parameter matrix $\cba{}$ as follows,
\begin{align}
\y_t&=\Cb\x_t+\Db\ub_t+\z_t,\nn\\
&=\Cb(\A\x_{t-1}+\B\ub_{t-1}+\w_{t-1})+\Db\ub_t+\z_t,\nn\\
&=\Cb\A^{\TK-1}\x_{t-\TK+1}+\sum_{i=1}^{\TK-1}\Cb\A^{i-1}\B\ub_{t-i}+\sum_{i=1}^{\TK-1}\Cb\A^{i-1}\w_{t-i}+\Db\ub_t+\z_t,\nn\\
&=\cba{}\ubb_t+\ca{}\wb_t+\z_t+\eb_t,\label{lsp eq}
\end{align}
where, $\eb_t=\Cb\A^{\TK-1}\x_{t-\TK+1}$ corresponds to the error due to the effect of the state at time $t-\TK+1$.
With this relation, we will use $(\ubb_t,\y_t)_{t=\TK}^{\bN}$ as inputs and outputs of our regression problem. We treat $\wb_t$, $\z_t$, and $\eb_t$ as additive noise and attempt to estimate $\cba{}$ from covariates $\ubb_t$. Note that, the noise terms are zero-mean including $\eb_t$ since we assumed $\x_1=0$. With these in mind, we form the following least-squares problem,
\[
\cbe{}=\underset{{\X\in\R^{m\times \TK p}}}{\arg\min}~\sum_{t=\TK}^{\bN} \tn{\y_t-\X\ubb_t}^2.
\]
Defining our label matrix $\Y$ and input data matrix $\Ub$ as,
\begin{align}
\Y=[\y_{\TK},~\y_{\TK+1},~\dots,~\y_{\bN}]^*\in\R^{N\times m}~~~\text{and}~~~\Ub=[\ubb_{\TK},~\ubb_{\TK+1},~\dots,~\ubb_{\bN}]^*\in\R^{N\times \TK p},\label{label and data}
\end{align}
we obtain the minimization $\min_{\X}\tf{\Y-\Ub\X^*}^2$. Hence, the least-squares solution $\hat{\cba{}}$ is given by
\begin{align}
\cbe{}=(\Ub^{\dagger}\Y)^*,\label{pinv est}
\end{align}
where $\Ub^{\dagger}=(\Ub^*\Ub)^{-1}\Ub^*$ is the left pseudo-inverse of $\Ub$. Ideally, we would like the estimation error $\tf{\cba{}-\cbe{}}^2$ to be small. Our main result bounds the norm of the error as a function of the sample size $N$ and noise levels $\sigma_w$ and $\sigma_z$. 

\section{Results on Learning Markov Parameters}
Let $\rho(\cdot)$ denote the spectral radius of a matrix which is the largest absolute value of its eigenvalues. Our results in this section apply to stable systems where $\rho(\A)<1$. Additionally we need a related quantity involving $\A$ which is the spectral norm to spectral radius ratio of its exponents defined as $\Phi(\A)=\sup_{\tau\geq 0} \frac{\|\A^\tau\|}{\rho(\A)^\tau}$. We will assume $\Phi(\A)<\infty$ which is a mild condition: For instance, if $\A$ is diagonalizable, $\Phi(\A)$ is a function of its eigenvector matrix and is finite. 
Another important parameter is the steady state covariance matrix of $\x_t$ which is given by
\[
\Ginf=\sum_{i=0}^\infty\sigma_w^2 \A^i(\A^*)^i+\sigma_u^2\A^i\B\B^*(\A^*)^i.
\]
It is rather trivial to show that for all $t\geq 1$, $\bSi(\x_t)\preceq \Ginf$. We will use $\Ginf$ to bound the error $\e_t$ due to the unknown state at time $t-\TK+1$. Following the definition of $\e_t$, we have that $\|\bSi(\e_t)\|\leq  \|\Cb\A^{\TK-1}\|^2\|\Ginf\|$. We characterize the impact of $\e_t$ by its ``effective standard deviation'' $\sigma_e$ that is obtained by scaling the bound on $\sqrt{\|\bSi(\e_t)\|}$ by an additional factor $\Phi(\A)\sqrt{T/(1-\rho(\A)^{2\TK})}$ which yields,
\begin{align}
\sigma_e=\Phi(\A)\|\Cb\A^{\TK-1}\|\sqrt{\frac{\TK\|\Ginf\|}{1-\rho(\A)^{2\TK}}}.\label{sigma e}
\end{align}
Our first result is a simplified version of Theorem \ref{main thm} and captures the problem dependencies in terms of the {\em{total standard deviations}} $\sigma_z+\sigma_e+\sigma_w\|\ca{}\|$ and the {\em{total dimensions}} $m+p+n$.
\begin{theorem} \label{main thm simple}Suppose $\rho(\A)^\TK\leq 0.99$ and $N\geq N_0=c\TK q \log^2(2\TK q)\log^2(2N q)$ where $q=p+n+m$. Given observations of a single trajectory until time $\bN=N+\TK-1$, with high probability\footnote{Precise statement on the probability of success is provided in the proof}, the least-square estimator of the Markov parameter matrix obeys
\[
\|\cbe{}-\cba{}\|\leq \frac{{\sigma_z+\sigma_e+\sigma_w\|\ca{}\|}}{\sigma_u}\sqrt{\frac{N_0}{N}}.
\]
\end{theorem}
\noindent{\bf{Remark:}} Our result is stated in terms of the spectral norm error $\|\cbe{}-\cba{}\|$. One can deduce the following Frobenius norm bound by naively bounding $\sigma_e,\sigma_z$ terms and swapping $\|\ca{}\|$ term by $\tf{\ca{}}$ (following \eqref{pinv terms}, \eqref{error decomp}). This yields, $\tf{\cbe{}-\cba{}}\leq \frac{(\sigma_z+\sigma_e)\sqrt{m}+\sigma_w\tf{\ca{}}}{\sigma_u}\sqrt{\frac{N_0}{N}}$.

Our bound individually accounts for the the process noise sequence $\{\w_\tau\}_{\tau=t-\TK+1}^t$, measurement noise $\z_t$, and the contribution of the unknown state $\x_{t-\TK+1}$. Setting $\sigma_w$ and $\sigma_z$ to $0$, we end up with the unknown state component $\sigma_e$. $\sigma_e$ has a $\|\Cb\A^{\TK-1}\|$ multiplier inside hence larger $\TK$ implies smaller $\sigma_e$. On the other hand, larger $\TK$ increases the size of the $\cba{}$ matrix as its dimensions are $m\times \TK p$. This dependence is contained inside the $N_0$ term which grows proportional to $\TK p$ (ignoring $\log$ terms). $\TK p$ corresponds to the minimum observation period since there are $m\TK p$ unknowns and we get to observe $m$ measurements at each timestamp. Hence, ignoring logarithmic terms, our result requires $N\gtrsim\TK p$ and estimation error decays as $\sqrt{\TK p/N}$. This behavior is similar to what we would get from solving a linear regression problem with independent noise and independent covariates \cite{friedman2001elements}. This highlights the fact that our analysis successfully overcomes the dependencies of covariates and noise terms.

Our main theorem is a slightly improved version of Theorem \ref{main thm simple} and is stated below. Theorem \ref{main thm simple} is operational in the regime $N\gtrsim\TK (p+m+n)$. In practical applications, hidden state dimension $n$ can be much larger than number of sensors $m$ and input dimension $p$. On the other hand, the input data matrix $\Ub$ becomes tall as soon as $N\geq \TK p$ hence ideally \eqref{pinv est} should work as soon as $N\gtrsim \TK p$. Our main result shows that reliable estimation is indeed possible in this more challenging regime. It also carefully quantifies the contribution of each term to the overall estimation error.
\begin{theorem} \label{main thm}Suppose system is stable (i.e.~$\rho(\A)<1$) and $N\geq c\TK p \log^2(2\TK p)\log^2(2N p)$. We observe a trajectory until time $\bN=N+\TK-1$. Then, with high probability, the least-square estimator of the Markov parameter matrix obeys
\begin{align}
\|\cbe{}-\cba{}\|\leq \frac{R_w+R_e+R_z}{\sigma_u\sqrt{N}},\label{the bound}
\end{align}
where $R_w,R_e,R_z$ are given by
\begin{align*}
&R_z={8\sigma_z\sqrt{\TK p+m}},\\
&R_w={\sigma_w\|\ca{}\|\max\{\sqrt{\NW},\NW/\sqrt{N}\}},\\
&R_e={C\sigma_e\sqrt{(1+\frac{m\TK}{N(1-\rho(\A)^{\TK})})(\TK p+m)}}.
\end{align*}
Here $c,C>0$ are absolute constants and $\NW=c\TK q \log^2(2\TK q)\log^2(2Nq)$ where $q=p+n$. 
\end{theorem}
One can obtain Theorem \ref{main thm simple} from Theorem \ref{main thm} as follows. When $N\geq N_0\geq N_w$: $R_w$ satisfies $R_w\leq \sigma_w\|\F\|\sqrt{N_w}\leq \sigma_w\|\F\|\sqrt{N_0}$. Similarly, when $\rho(\A)^{\TK}$ is bounded away from $1$ by a constant and $N\geq N_0\geq \order{\TK m}$: $R_e$ satisfies $R_e\leq 2C\sigma_e\sqrt{\TK p+m}\leq \sigma_e\sqrt{N_0}$.

One advantage of Theorem \ref{main thm} is that it works in the regime $\TK p\lesssim N\lesssim \TK(p+n+m)$. Additionally, Theorem \ref{main thm} provides tighter individual error bounds for the $\sigma_z,\sigma_w,\sigma_e$ terms and explicitly characterizes the dependence on $\rho(\A)$ inside the $R_e$ term.

Theorem \ref{main thm} can be improved in a few directions. Some of the log factors that appear in our sample size might be spurious. These terms are arising from a theorem borrowed from Krahmer et al. \cite{krahmer2014suprema}; which actually has a stronger implication than what we need in this work. We also believe \eqref{sigma e} is overestimating the correct dependence by a factor of $\sqrt{\TK}$. 
\subsection{Estimating the Output via Markov Parameters}
The following lemma illustrates how learning Markov parameters helps us bound the prediction error. 
\begin{lemma} [Estimating $\y_{\TK}$] Suppose $\x_1=0$ and $\z_t\sim\Nn(0,\sigma^2_z\Iden)$, $\ub_t\sim\Nn(0,\sigma^2_u\Iden)$, $\w_t\sim\Nn(0,\sigma^2_w\Iden)$ for $t\geq 0$ as described in Section \ref{sec setup}. Assume, we have an estimate $\cbe{}$ of $\cba{}$ that is independent of these variables and we employ the $\y_t$ estimator
\[
\hat{\y}_t=\cbe{}\ubb_t.
\]
Then, 
\[
\E[\tn{\y_t-\yh_{t}}^2]\leq \sigma_w^2\tf{\ca{}}^2+\sigma_u^2\tf{\cba{}-\cbe{}}^2+m\sigma^2_z+\|\Cb\A^{\TK-1}\|^2\tr{\Ginf}.
\]
\end{lemma}
\begin{proof} Following from the input/output identity \eqref{lsp eq}, the key observation is that for a fixed $t$, $\ubb_t,\wb_t,\z_t,\e_t$ are all independent of each other and their prediction errors are uncorrelated. Since $\ubb_t\sim\Nn(0,\sigma_u^2\Iden)$, $\E[\tn{(\cba{}-\cbe{})\ubb}^2]=\sigma_u^2\|\cba{}-\cbe{}\|_F^2$. Same argument applies to $\wb\sim\Nn(0,\sigma_w^2\Iden),\z_t\sim\Nn(0,\sigma^2_z\Iden)$ and $\e_t$ which obeys $\E[\tn{\e_t}^2]=\tr{\bSi(\e_t)}$. Observe that $i$th largest eigenvalue $\lambda_i(\bSi(\e_t))$ of $\bSi(\e_t)$ is upper bounded by $\|\Cb\A^{\TK-1}\|^2\lambda_i(\bSi(\x_{t-T+1}))$ via Min-Max principle \cite{horn1990matrix} hence $\E[\tn{\e_t}^2]\leq \|\Cb\A^{\TK-1}\|^2\tr{\bSi(\x_{t-\TK+1})}$ $\leq  \|\Cb\A^{\TK-1}\|^2\tr{\Ginf}$. 
\end{proof}

\section{Markov Parameters to Hankel Matrix:\\Low Order Approximation of Stable Systems}\label{sec low ord}
So far our attention has focused on estimating the impulse response $\cba{}$ for a particular horizon $\TK$. Clearly, we are also interested in understanding how well we learn the overall behavior of the system by learning a finite impulse approximation. In this section, we will apply our earlier results to approximate the overall system by using as few samples as possible. A useful idea towards this goal is taking advantage of the stability of the system. The Markov parameters decay exponentially fast if the system is stable i.e.~$\rho(\A)<1$. This means that, most of the Markov parameters will be very small after a while and not learning them might not be a big loss for learning the overall behavior. In particular, $\tau$'th Markov parameter obeys
\[
\|\Cb\A^\tau\B\|\leq\Phi(\A)\rho(\A)^\tau \|\Cb\|\|\B\|.
\]
This implies that, the impact of the impulse response terms we don't learn can be upper bounded. For instance, the total spectral norm of the tail terms obey
\begin{align}
\sum_{\tau=\TK-1}^\infty\|\Cb\A^\tau\B\|\leq\sum_{\tau=\TK-1}^\infty\Phi(\A)\rho(\A)^\tau \|\Cb\|\|\B\| \leq \frac{\Phi(\A) \|\Cb\|\|\B\|\rho(\A)^{\TK-1}}{1-\rho(\A)}.\label{tail spec}
\end{align}
To proceed fix a finite horizon $K$ that will later be allowed to go infinity. Represent the estimate $\cbe{}$ as $[\hat{\Db},~\cbe{0},~\dots~\cbe{T-2}]$ where $\cbe{i}$ corresponds to the noisy estimate of $\Cb\A^{i}\B$. Now, let us consider the estimated and true order $K$ Markov parameters
\begin{align*}
&\Gb^{(K)}=[\hat{\Db},~\cbe{0},~\dots~\cbe{T-2}~0~\dots~0]\\
&\hat{\Gb}^{(K)}=[{\Db},~\Cb\B,~\Cb\A\B~\dots~\Cb\A^{K-2}\B].
\end{align*}
Similarly we define the associated $K\times K$ block Hankel matrices of size $mK\times pK$ as follows
\begin{align}
&\hat{\Hb}^{(K)}=\begin{bmatrix}\hat{\Db}&\cbe{0}&\dots&\cbe{T-3}&\cbe{T-2}&0&\dots&0\\
\cbe{0}&\cbe{1}&\dots&\cbe{T-2}&0&0&\dots&0\\
&&&\vdots&&&&\\
\cbe{T-3}&\cbe{T-2}&\dots&0&0&0&\dots&0\\
\cbe{T-2}&0&\dots&0&0&0&\dots&0\\
\vdots\\
0&&&\dots&&&&0\end{bmatrix}\quad\quad
\Hb^{(K)}=\begin{bmatrix}\Db&\Cb\B&\dots&\Cb\A^{K-2}\B\\
\Cb\B&\Cb\A\B&\dots&\Cb\A^{K-1}\B\\
&&\vdots&\\
\Cb\A^{K-2}\B&\Cb\B&\dots&\Cb\A^{2K-3}\B
\end{bmatrix}\label{hmatrices}
\end{align}
The following theorem merges results of this section with a specific choice of $\TK$ to give approximation bounds for the infinite Markov operator $\Gb^{(\infty)}$ and Hankel operator $\Hb^{(\infty)}$. For notational simplicity, we shall assume that there is no process noise.
\begin{theorem} \label{truncate thm} Suppose the spectral radius obeys $\rho(\A)<1$. Fix a number $1>\eps_0>0$ and suppose process noise obeys $\sigma_w=0$. Assume sample size $N$ and estimation horizon $\TK$ satisfies\footnote{Exact form of the bounds depend on $\A,\B,\Cb$ and is provided in the proof.}
\begin{align}
&{N}{}\geq cTp\log^2(2\TK p)\log^2N\nn\\
&\TK\geq  \frac{c_0+\log (N/T+T(1+m/p))-\log \eps_0}{-\log\rho(\A)}.\label{tk cond}
\end{align}
Then, given observations of a single trajectory until time $\bN=N+\TK-1$ and estimating first $\TK$ Markov parameters via least-squares estimator \eqref{pinv est}, with high probability, the following bounds hold on the infinite impulse response and Hankel matrix of the system.
\begin{align}
&\|\Gb^{(\infty)}-\hat{\Gb}^{(\infty)}\|\leq (8\frac{\sigma_z}{\sigma_u}+\eps_0) \sqrt{\frac{Tp+m}{N}}\nn\\
&\|\Hb^{(\infty)}-\hat{\Hb}^{(\infty)}\|\leq \TK(8\frac{\sigma_z}{\sigma_u}+\eps_0) \sqrt{\frac{Tp+m}{N}}.\nn
\end{align}
\end{theorem}

In essence, the above theorem is a corollary of Theorem \ref{main thm}. However, it further simplifies the bounds and also provides approximation to systems overall behavior (e.g.~infinite Hankel matrix). In particular, these bounds exploit stability of the system and allows us to treat the system as if it has a logarithmic order. Observe that \eqref{tk cond} only logarithmically depends on the critical problem variables such as precision $\eps_0$ and spectral radius. In essence, the effective system order is dictated by the eigen-decay and equal to $\TK\sim \order{-\frac{1}{\log(\rho(\A))}}$ hence stability allows us to treat the system as if it has a logarithmically small order. Ignoring logarithmic terms except $\rho(\A)$, using $\eps_0,\sigma_z/\sigma_u=\order{1}$ and picking
\[
T=\order{\frac{-1}{\log(\rho(\A))}}\quad\text{and}\quad N=\order{\delta^{-2}(\TK p+m)},
\]
guarantees
\[
\|\Gb^{(\infty)}-\hat{\Gb}^{(\infty)}\|\leq \delta\quad\text{and}\quad \|\Hb^{(\infty)}-\hat{\Hb}^{(\infty)}\|\leq \order{\frac{-\delta}{\log(\rho(\A))}}.
\]
Remarkably, sample size is independent of the state dimension $n$ and only linearly grows with $p$. Indeed, one needs at least $\order{p}$ samples to estimate a single Markov parameter and we need only logarithmically more than this minimum (i.e.~$N\approx \frac{ -\order{p}}{\log(\rho(\A))}$) to estimate the infinite Hankel matrix.




\section{Non-Asymptotic System Identification via Ho-Kalman}\label{sec ho-kal}

\begin{algorithm} [!t]\caption{Ho-Kalman Algorithm to find a State-Space Realization.}\label{algo 1}
\begin{algorithmic}[1]
\Procedure{Ho-Kalman Minimum Realization}{}
\item {\bf{Inputs:}} \rrr{Length} $\TK$, Markov parameter matrix estimate $\cbe{}$, system order $n$,

\hspace{14pt}Hankel shape $(\TK_1,\TK_2+1)$ with $\TK_1+\TK_2+1=\TK$.
\item {\bf{Outputs:}} State-space realization $\hat{\A},\hat{\B},\hat{\Cb}$.
\State Form the Hankel matrix $\hat{\Hb}\in\R^{m \TK_1\times p(\TK_2+1)}$ from $\cbe{}$.
\State $\Hbhm\in\R^{m \TK_1\times p\TK_2}\gets \text{first-$p\TK_2$-columns-of}(\hat{\Hb})$.
\State $\Hbhl\in \R^{m \TK_1\times p\TK_2}\gets \text{rank-$n$-approximation-of}(\Hbhm)$.
\State $\Ub,\bSi,\V=\text{SVD}(\Hbhl)$.
\State $\Obh\in \R^{m \TK_1\times n}\gets \Ub\bSi^{1/2}$.
\State $\Qbh\in \R^{n\times p \TK_2}\gets \bSi^{1/2}\V^*$.
\State $\Ch\gets \text{first-$m$-rows-of}(\Obh)$.
\State $\Bh\gets \text{first-$p$-columns-of}(\Qbh)$.
\State $\Hbhp\in\R^{m \TK_1\times p\TK_2}\gets \text{last-$p\TK_2$-columns-of}(\hat{\Hb})$.
\State $\Ah\gets \Obh^{\dagger}\Hbhp\Qbh^{\dagger}$.\\
\Return $\hat{\A}\in\R^{n\times n},\hat{\B}\in\R^{n\times p},\hat{\Cb} \in\R^{m\times n}$.
\EndProcedure
\end{algorithmic}
\end{algorithm}

In this section, we first describe the \HK~algorithm \cite{ho1966effective} that generates $\A,\B,\Cb,\Db$ from the Markov parameter matrix $\Gb$. We also show that the algorithm is stable to perturbations in $\Gb$ and the output of \HK~gracefully \rrr{degrades} as a function of $\|\cba{}-\cbe{}\|$. Combining this with Theorem \ref{main thm simple} implies {\em{guaranteed}} non-asymptotic identification of multi-input-multi-output systems from a {\em{single trajectory}}. We remark that results of this section do not assume stability and applies to arbitrary, possibly unstable, systems. We will use the following Hankel matrix definition to introduce the algorithms.

\begin{definition}[Clipped Hankel matrix] \label{hankel def} Given a block matrix $\X=[\X_1,~\X_{2},~\dots~\X_{\TK}]\in\R^{m\times \TK p}$ and integers $\TK_1,\TK_2$ satisfying $\TK_1+\TK_2\leq \TK$, define the associated $(\TK_1,\TK_2)$ Hankel matrix $\Hb=\Hb(\X)\in\R^{\TK_1 m\times \TK_2p}$ to be the $\TK_1\times \TK_2$ block matrix with $m\times p$ size blocks where $(i,j)$th block is given by
\[
\Hb[i,j]=\X_{i+j}.
\]
\end{definition}
Note that, $\Hb$ does not contain $\X_{1}$, which shall correspond to the $\Db$ (or $\hat{\Db}$) matrix for our purposes. This is solely for notational convenience as the first Markov parameter in $\cba{}$ is $\Db$; however $\A,\B,\Cb$ are identified from the remaining Markov parameters of type $\Cb\A^{i}\B$.
\subsection{System Identification Algorithm}
Given a noisy estimate $\cbe{}$ of $\cba{}$, we wish to learn good system matrices $\Ah,\Bh,\Ch,\hat{\Db}$ from $\cbe{}$ up to trivial ambiguities. This will be achieved by using Algorithm \ref{algo 1} which admits the matrix $\cbe{}$, system order $n$ and Hankel dimensions $\TK_1,\TK_2$ as inputs. \rrr{Throughout this section, we make the following two assumptions to ensure that the system we wish to learn is order-$n$ and our system identification problem is well-conditioned.}
\begin{itemize}
\item the system is observable and controllable; hence $n>0$ is the order of the system. 
\item $(\TK_1,\TK_2)$ Hankel matrix $\Hb(\cba{})$ formed from $\cba{}$ is rank-$n$. This can be ensured by choosing sufficiently large $\TK_1,\TK_2$. In particular $\TK_1\geq n,\TK_2\geq n$ is guaranteed to work {\color{black}by the first assumption above}.
\end{itemize}
Learning state-space representations is a non-trivial, inherently non-convex problem. Observe that there are multiple state-space realizations that \rrr{yields the same system and Markov matrix $\cba{}$}. In particular, for any nonsingular matrix $\Tb\in\R^{n\times n}$, 
\[
\A'=\Tb^{-1}\A\Tb,~\B'=\Tb^{-1}\B,~\Cb'=\Cb\Tb,
\]
is a valid realization \rrr{and yields the same system}. Hence, similarity transformations of $\A,\B,\Cb$ generate a class of solutions. Note that $\Db$ is already estimated as part of $\cba{}$. Since $\Db$ is a submatrix of $\Gb$, we clearly have
\[
\|\Db-\hat{\Db}\|\leq \|\cba{}-\cbe{}\|.
\]
Hence, we focus our attention on learning $\A,\B,\Cb$. Suppose we have access to the true Markov parameters $\cba{}$ and the corresponding $(\TK_1,\TK_2+1)$ Hankel matrix $\Hb(\cba{})$. In this case, $\Hb$ is a rank-$n$ matrix and $(i,j)$th block of $\Hb$ is equal to $\Cb\A^{i+j-2}\B$. Defining (extended) controllability and observability matrices $\Qb=[\B,~\A\B,~\dots~\A^{\TK_2}\B]$ and $\Ob=[\Cb^*,~(\Cb\A)^*,$ $~\dots~(\Cb\A^{\TK_1-1})^*]^*$, we have $\Hb=\Ob\Qb$. However, it is not clear how to find $\Ob,\Qb$. 

{\color{black}
The \HK~algorithm accomplishes this task by finding a balanced realization and returning {\em{some}} $\Ah,\Bh,\Ch$ matrices from possibly noisy Markov parameter matrix $\cbe{}$. Let the input to the algorithm be $\cbe{}=[\hat{\Db},~\cbe{0},~\dots~\cbe{T-2}]$ where $\cbe{i}$ corresponds to the noisy estimate of $\Cb\A^{i}\B$. We construct the $(\TK_1,\TK_2+1)$ Hankel matrix $\hat{\Hb}$ as described above so that $(i,j)$th block of $\hat{\Hb}$ is equal to $\cbe{i+j-2}$. Let $\Hbhm\in\R^{m\TK_1\times p\TK_2}$ be the submatrix of $\hat{\Hb}$ after discarding the rightmost ${m\TK_1\times p}$ block and $\Hbh$ be the best rank-$n$ approximation of $\Hbhm$ obtained by setting its all but top $n$ singular values to zero. Let $\Hbhp$ be the submatrix after discarding the left-most ${m\TK_1\times p}$ block. Note that both $\Hbh,\Hbhp$ have size $\R^{m\TK_1\times p\TK_2}$. Take the singular value decomposition (SVD) of the rank-$n$ matrix $\Hbhl$ as $\Hbhl=\Ub\bSi\V^*$ (with $\bSi\in\R^{n\times n}$) and write
\[
\Hbhl=(\Ub\bSi^{1/2})\bSi^{1/2}\V^*={\Obh}{\Qbh}.
\]
If $\cbe{}$ was equal to the ground truth $\cba{}$, then ${\Obh},{\Qbh}$ would correspond to the order $\TK_1$ observability matrix $\bar{\Ob}=\Ub\bSi^{1/2}$ and the order $\TK_2$ controllability matrix ${\bar{\Qb}}=\bSi^{1/2}\V^*$ of the actual balanced realization based on {\em{noiseless}} SVD. Here, $\bar{\Ob},\bar{\Qb}$ matrices are not necessarily equal to $\Ob,\Qb$, \rrr{however they yield the same system}. Note that, the columns of $\hat{\Ob},\hat{\Qb}$ are the scaled versions of the left and right singular vectors of $\Hbhl$ respectively. The \HK~algorithm finds $\Ah,\Bh,\Ch$ as follows.
\vspace{-.1cm}
\begin{itemize}
\item $\Ch$ is the first $m\times n$ submatrix of $\Obh$. 
\item $\Bh$ is the first $n\times p$ submatrix of $\Qbh$. 
\item $\Ah=\Obh^\dagger \Hbhp \Qbh^\dagger$.
\vspace{-.1cm}
\end{itemize}
This procedure (\HK) returns the true balanced realization of the system when Markov parameters are known i.e.~$\cbe{}=\cba{}$. Our goal is to show that even with noisy Markov parameters, this procedure returns good estimates of the \rrr{true} balanced realization.} We remark that there are variations of this procedure; however the core idea is the same and they are equivalent when the true Markov parameters are used as input. For instance, when constructing $\hat{\Hb}$, one can attempt to improve the noise robustness of the algorithm by picking balanced dimensions $m\TK_1\approx p\TK_2$.

\subsection{Robustness of the Ho-Kalman Algorithm}
Observe that $\hat{\Hb},\Hbhm,\Hbhl,\Hbhp,\hat{\Ob},\hat{\Qb}$ of Algorithm \ref{algo 1} are functions of the input matrix $\cbe{}$. For the subsequent discussion, we let
\begin{itemize}
\item $\Hb,\Hb^-,\Hbl,\Hb^+,\Ob,\Qb$ be the matrices corresponding to ground truth $\cba{}$.
\item $\hat{\Hb},\Hbhm,\Hbhl,\Hbhp,\hat{\Ob},\hat{\Qb}$ be the matrices corresponding to the estimate $\cbe{}$.
\end{itemize}
Furthermore, let $\Ab,\Bb,\Cbb$ be the actual balanced realization associated with $\cba{}$ and let $\Ah,\Bh,\Ch$ be the \HK~output associated with $\cbe{}$. Note that $\Hbl=\Hb^-$ since $\Hb^-$ is already rank $n$. We now provide a lemma relating the estimation error of $\cba{}$ to that of $\Hbl$ and $\Hb$.
\begin{lemma}\label{g to h and l} $\Hb,\hat{\Hb}$ and $\Hbl,{\Hbh}$ satisfies the following perturbation bounds,
\begin{align}
&\max\{\|\Hb^+-\hat{\Hb}^+\|,\|\Hb^--\hat{\Hb}^-\|\}\leq \|\Hb-\hat{\Hb}\|\leq \sqrt{\min\{\TK_1,\TK_2+1\}}\|\cba{}-\cbe{}\|,\label{cba upp}\\
&\|\Hbl-{\Hbh}\|\leq 2\|\Hb^--\hat{\Hb}^-\|\leq  2\sqrt{\min\{\TK_1,\TK_2\}}\|\cba{}-\cbe{}\|.\label{cba upp2}
\end{align}
\end{lemma}
Let us denote the $n$th largest singular value of $\Lb$ via $\smin$. Note that $\smin$ is the smallest nonzero singular value of $\Lb$ since $\text{rank}(\Lb)=n$. \rrr{A useful implication of Theorem \ref{main thm simple} (in light of Lemma \ref{g to h and l}) is that  if $\smin$ is large enough, the true system order $n$ can be non-asymptotically estimated from the noisy Markov parameter estimates via singular value thresholding.}

Our next result shows the robustness of the \HK~algorithm to possibly adversarial perturbations on the Markov parameter matrix $\cba{}$.

\begin{theorem}\label{kh stable}  Suppose $\Hb$ and $\hat{\Hb}$ be the Hankel matrices derived from $\cba{}$ and $\cbe{}$ respectively per Definition \ref{hankel def}. Let $\Ab,\Bb,\Cbb$ be the state-space realization corresponding to the output of \HK~with input $\cba{}$ and $\Ah,\Bh,\Ch$ be the state-space realization corresponding to output of \HK~with input $\cbe{}$. Suppose the system $\A,\B,\Cb,\Db$ is {\em{observable and controllable}} and let $\Ob,\Qb$ and $\hat{\Ob},\hat{\Qb}$ be order-$n$ controllability/observability matrices associated with $\cba{}$ and $\cbe{}$ respectively. Suppose $\smin>0$ and perturbation obeys
\begin{align}
\|\Hbl-\Hbh\|\leq \smin/2.\label{perturb req}
\end{align}
Then, there exists a unitary matrix $\Tb\in\R^{n\times n}$ such that,
\begin{align}
&\|\Cbb-\Ch\Tb\|_F\leq\tf{\Ob-\hat{\Ob}\Tb}\leq \sqrt{5n\|\Lb-\hat{\Lb}\|},\label{pair bounds}\\
&\|\Bb-\Tb^*\Bh\|_F\leq \tf{\Qb-\Tb^*\hat{\Qb}}\leq \sqrt{5n\|\Lb-\hat{\Lb}\|}.
\end{align}
Furthermore, hidden state matrices $\Ah,\Ab$ satisfy
\begin{align}
 \tf{\Ab-\Tb^*\Ah\Tb}\leq \frac{14\sqrt{n}}{\smin}(\sqrt{\frac{\|\Lb-\hat{\Lb}\|}{\smin}}(\|{\Hb}^+\|+\|\Hb^+-\hat{\Hb}^+\|)+\|\Hb^+-\hat{\Hb}^+\|).\label{a bound}
\end{align}
\end{theorem}

Above, $\|\Hb^+-\hat{\Hb}^+\|,\|\Lb-\hat{\Lb}\|$ are perturbation terms that can be bounded in terms of $\|\Hb-\hat{\Hb}\|$ or $\|\cba{}-\cbe{}\|$ via Lemma \ref{g to h and l}. This result shows that \HK~solution is robust to noise up to trivial ambiguities. \rrr{Robustness is controlled by $\smin$ which typically corresponds to the weakest mode of the system.} We remark that for reasonably large $\TK_2$ choice, we have $\smin\approx \sigma_{\min}(\Hb)$ as $\Lb=\Hb^-$ is obtained by discarding the last block column of $\Hb$ which is exponentially small in $\TK_2$.

Since the \HK~algorithm is based on SVD, having a good control over singular vectors is crucial for the proof. We do this by utilizing the perturbation results from the recent literature \cite{tu2015low}. While we believe our result has the correct dependency, it is in terms of Frobenius norm rather than spectral. Having a better spectral norm control over $\Ab,\Bb,\Cbb$ would be an ideal future improvement. 

A corollary to this result can be stated in terms of $\smin$ and Hankel matrices $\Hb,\hat{\Hb}$. The result below follows from an application of Lemma \ref{g to h and l}.
\begin{corollary} \label{cor hk}Consider the setup of Theorem \ref{kh stable} and suppose $\smin>0$ and
\[
\|\Hb-\hat{\Hb}\|\leq \smin/4.
\]
Then, there exists a unitary matrix $\Tb\in\R^{n\times n}$ such that,
\begin{align}
\max\{\|\Cbb-\Ch\Tb\|_F,\tf{\Ob-\hat{\Ob}\Tb},\|\Bb-\Tb^*\Bh\|_F, \tf{\Qb-\Tb^*\hat{\Qb}}\}\leq 5\sqrt{n\|\Hb-\hat{\Hb}\|}.\label{pair bounds 2}
\end{align}
Furthermore, hidden state matrices $\Ah,\Ab$ satisfy
\begin{align}
 \tf{\Ab-\Tb^*\Ah\Tb}\leq \frac{50\sqrt{n{\|\Hb-\hat{\Hb}\|}{}}\|{\Hb\|}}{\sigma_{\min}^{3/2}(\Lb)}.\nn
\end{align}
\end{corollary}
Recall from Lemma \ref{g to h and l} that $\|\Hb-\hat{\Hb}\|\leq \sqrt{\min\{\TK_1,\TK_2+1\}}\|\cba{}-\cbe{}\|$. Hence, combining Corollary \ref{cor hk} and Theorem \ref{main thm simple} provides non-asymptotic guarantees for end-to-end system identification procedure. Theorem \ref{main thm simple} finds a good Markov parameter estimate $\cbe{}$ from a small amount of data and Corollary \ref{cor hk} translates this $\cbe{}$ into a robust state-space realization $\Ah,\Bh,\Ch,\hat{\Db}$.

\begin{figure}[t!]
 \begin{subfigure}[b]{0.5\textwidth}
        \includegraphics[width=\textwidth]{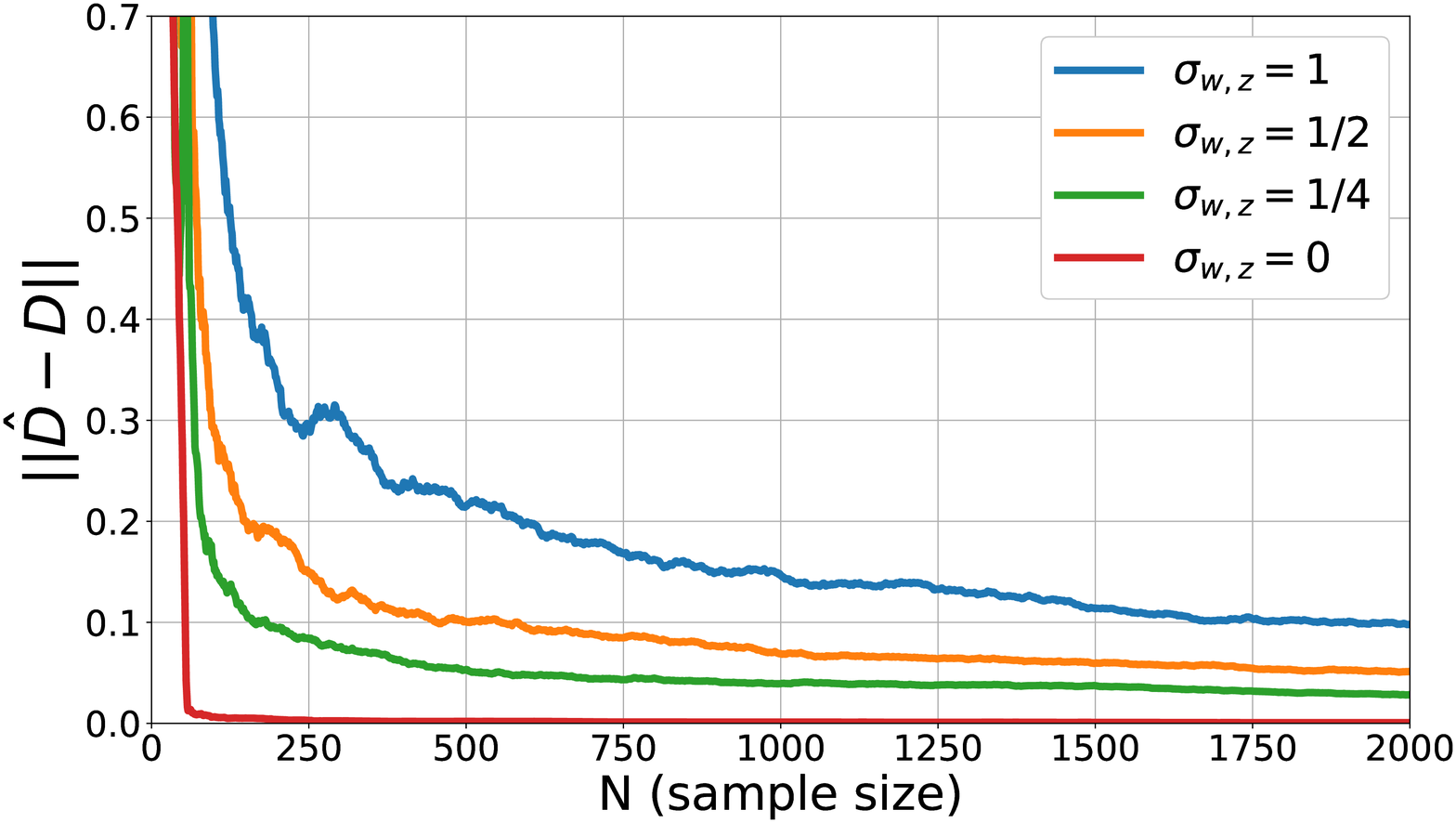}
    \end{subfigure} ~
    \begin{subfigure}[b]{0.5\textwidth}
        \includegraphics[width=\textwidth]{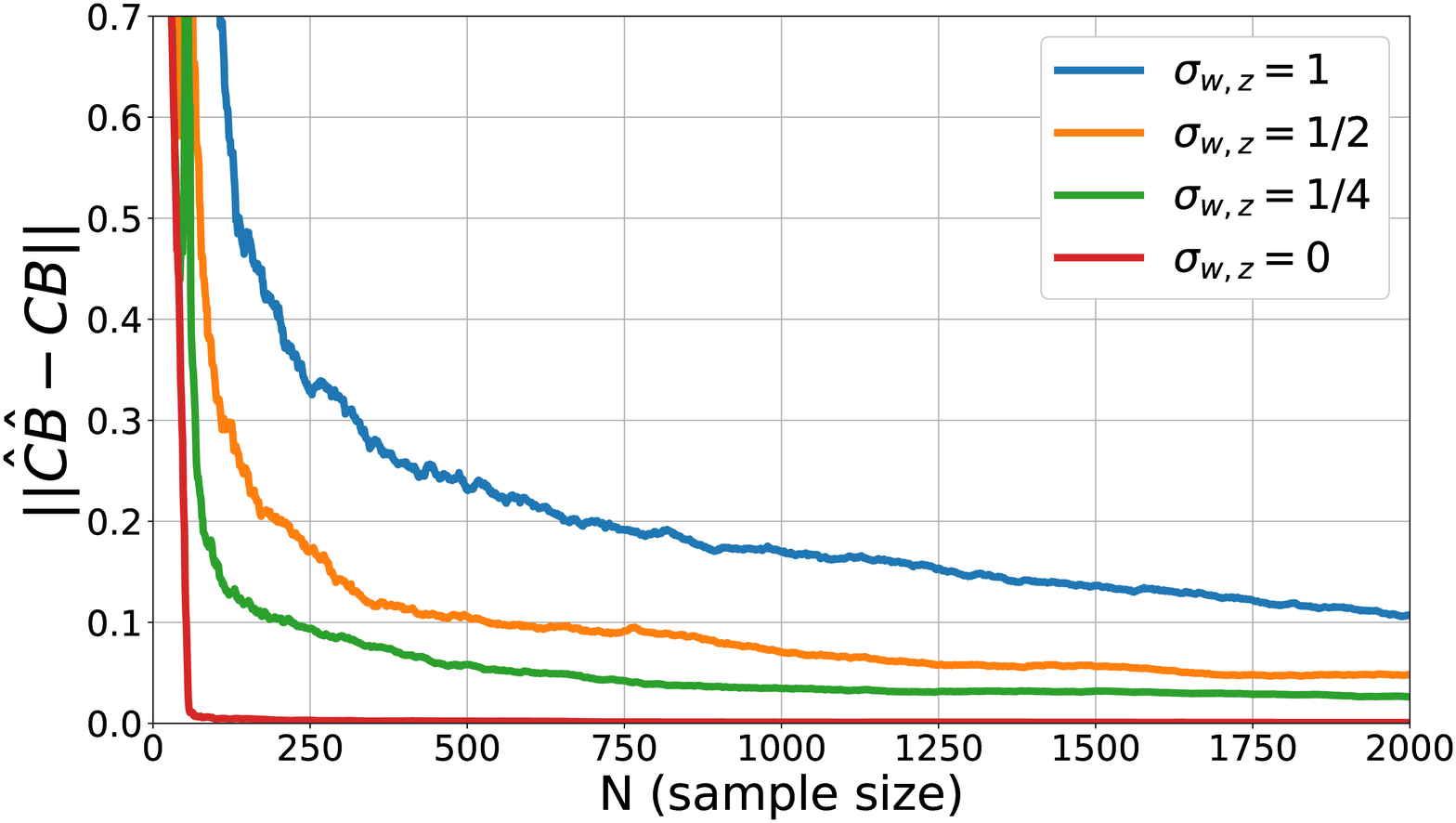}
    \end{subfigure}\\
    \begin{subfigure}[b]{0.5\textwidth}
        \includegraphics[width=\textwidth]{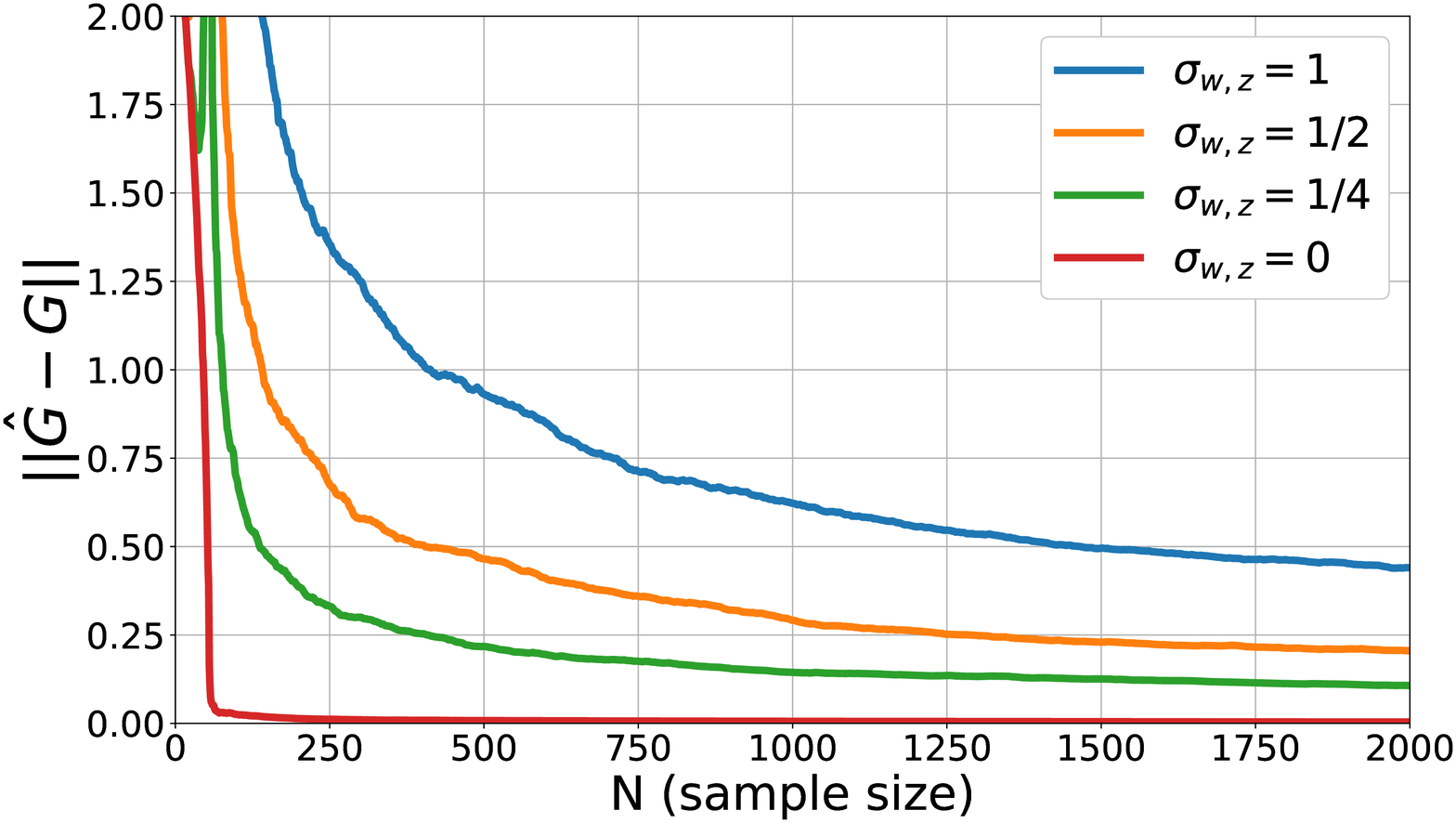}
    \end{subfigure} ~
    \begin{subfigure}[b]{0.5\textwidth}
        \includegraphics[width=\textwidth]{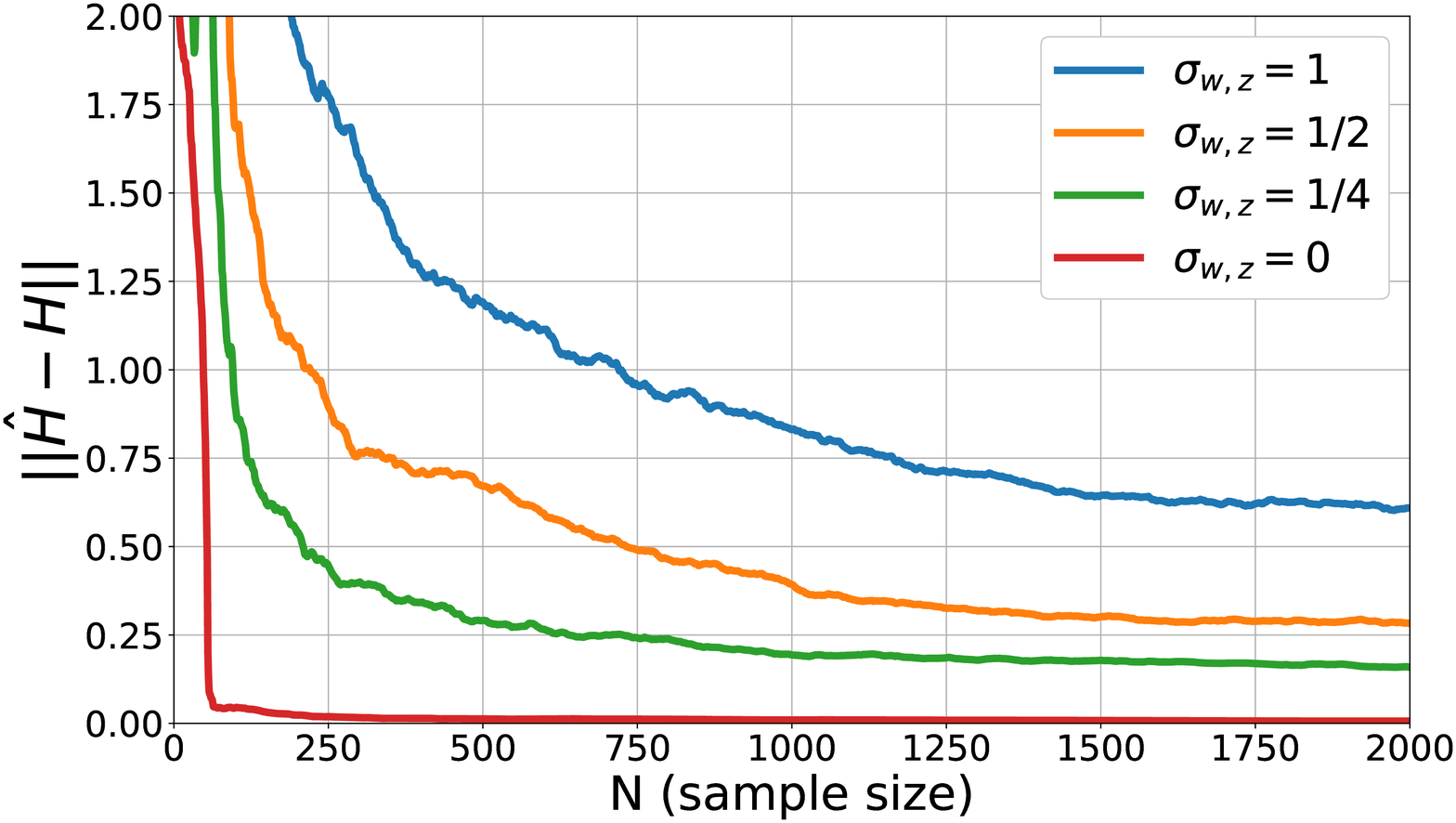}
    \end{subfigure}
    \caption{\small We consider the matrices that can directly be inferred from the Markov parameter matrix $\cba{}$. These are $\Db,\Cb\B$ which are the first two block submatrices of $\cba{}$, $\cba{}$ itself, and $\Hb$ which is the Hankel matrix that is constructed from blocks of $\cba{}$. These results are for $\TK=18$ which implies $\cba{}\in\R^{2\times 54}$ and $\Hb\in\R^{18\times 27}$ as we picked $\TK_1=\TK_2+1=9$.}
    \label{fig g}
\end{figure}

\section{Numerical Experiments}

We considered a MIMO (multiple input, multiple output) system with $m=2$ sensors, $n=5$ hidden states and input dimension $p=3$. To assess the typical performance of the least-squares and the \HK~algorithms, we consider random state-spaces as follows. We generate $\Cb,\Db$ with independent $\Nn(0,1/m)$ entries. We generate $\B$ with independent $\Nn(0,1/n)$ entries. These variance choices are to ensure these matrices are isometric in the sense that $\E[\tn{\M\vb}^2]=\tn{\vb}^2$ for a given vector $\vb$ and $\M\in\{\B,\Cb,\Db\}$. Hence, the impact of the standard deviations $\sigma_u,\sigma_w,\sigma_z$ are properly normalized. The input variance is fixed at $\sigma_u=1$ however noise variances will be modified during the experiments.

The most critical component of an LTI system is the $\A$ matrix. We picked $\A$ to be a diagonal matrix with its $n$ eigenvalues (i.e.~diagonal entries) are generated to be uniform random variables between $[0,0.9]$. The upper bound $0.9$ implies that we are working with stable matrices and the effect of unknown state vanishes for large $\TK$.

Finally, we conduct experiments for different $\TK$ values of $\TK\in\{6,12,18\}$. During \HK~procedure, we create a Hankel matrix $\hat{\Hb}$ of size $m\TK/2 \times p\TK/2$ and apply Algorithm \ref{algo 1}. Due to random generation of problem data,  even for $\TK=6$, the ground truth Hankel matrix $\Hb^-\in\R^{6\times 6}$ has rank $n=5$
so that \HK~procedure can indeed learn a good realization.

In our experimental setup, we pick a hyperparameter configuration of $\TK,\sigma_w,\sigma_z$ and generate a single rollout of the system until some time $t_{\infty}$. For each $\TK\leq \bN\leq t_{\infty}$, we solve the system via \eqref{pinv est} to obtain the estimate of $\cba{}$ and use Algorithm \ref{algo 1} to obtain a state-space realization $\Ah,\Bh,\Ch,\hat{\Db}$. The $x$-axis displays $N$ (which is the amount of available data at time $t=\bN$) and the $y$-axis displays the estimation error. Each curve in the figures is generated by averaging the outcomes of $20$ independent realizations of single trajectories.


In Figure \ref{fig g}, we considered the problem of estimating the matrices $\Db,\Cb\B,\cba{},\Hb$ when $\TK=18$. $\Db,\Cb\B$ are the first two impulse responses. Estimating $\cba{}$ and the associated Hankel matrix $\Hb$ helps verify our findings in Theorem \ref{main thm}. We plotted curves for varying noise levels $\sigma_w=\sigma_z\in\{0,1/4,1/2,1\}$. The main conclusion is that indeed estimation accuracy drastically improves as we observe the system for a longer period of time and collect more data. Note that estimation errors on $\Db$ and $\Cb\B$ are in the same ballpark. These are submatrices of $\cba{}$ hence their associated spectral norm errors are strictly lower compared to $\|\cba{}-\cbe{}\|$. \rrr{Per Definition \ref{hankel def}, $\Hb$ is constructed from the blocks of $\cba{}$ and its spectral norm error is in lines with $\cba{}$.} The other observation is that estimation error decays gracefully as a function of the noise levels for all matrices of interest. Since we picked a large $\TK$, the error due to unknown initial conditions (i.e.~$\e_t$) is fairly negligible. Hence when $\sigma_w=\sigma_z=0$, we quickly achieve near $0$ estimation error as the impact of the $\e_t$ term is also small.

\begin{figure}[t!]
 \begin{subfigure}[b]{0.5\textwidth}
        \includegraphics[width=\textwidth]{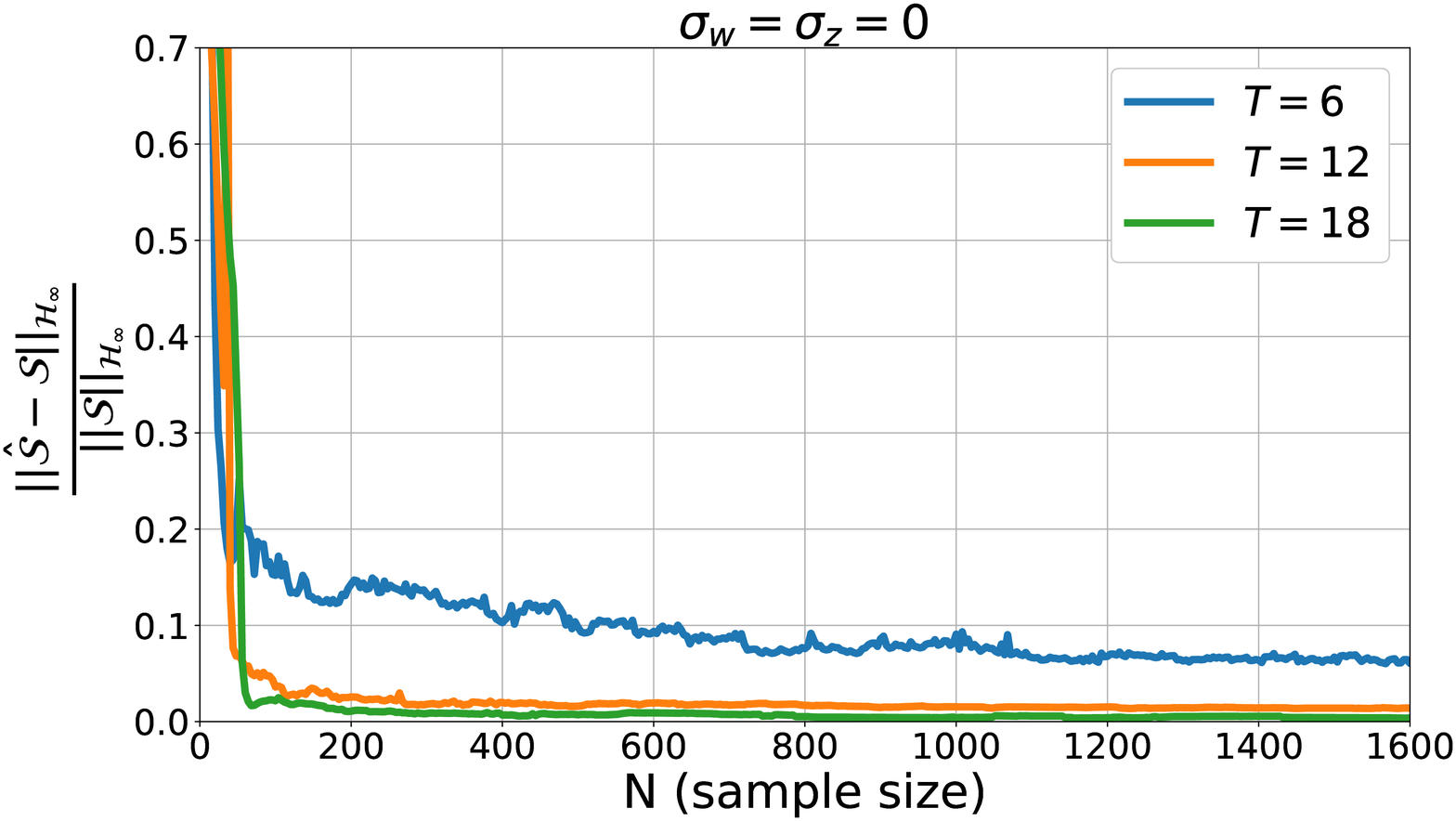}
        \caption{}
    \end{subfigure} ~
    \begin{subfigure}[b]{0.5\textwidth}
        \includegraphics[width=\textwidth]{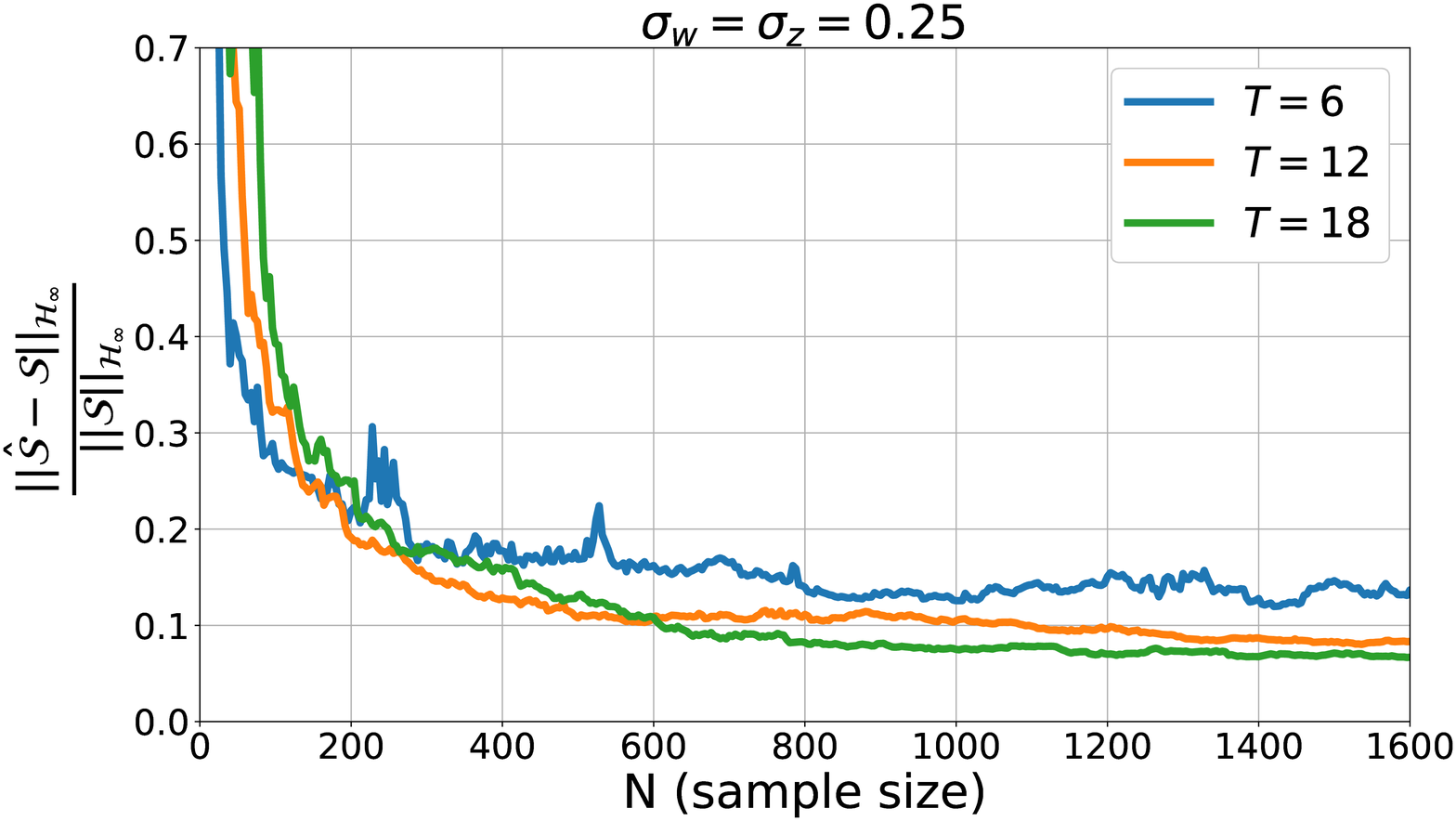}
        \caption{}
    \end{subfigure}\\
    \begin{subfigure}[b]{0.5\textwidth}
        \includegraphics[width=\textwidth]{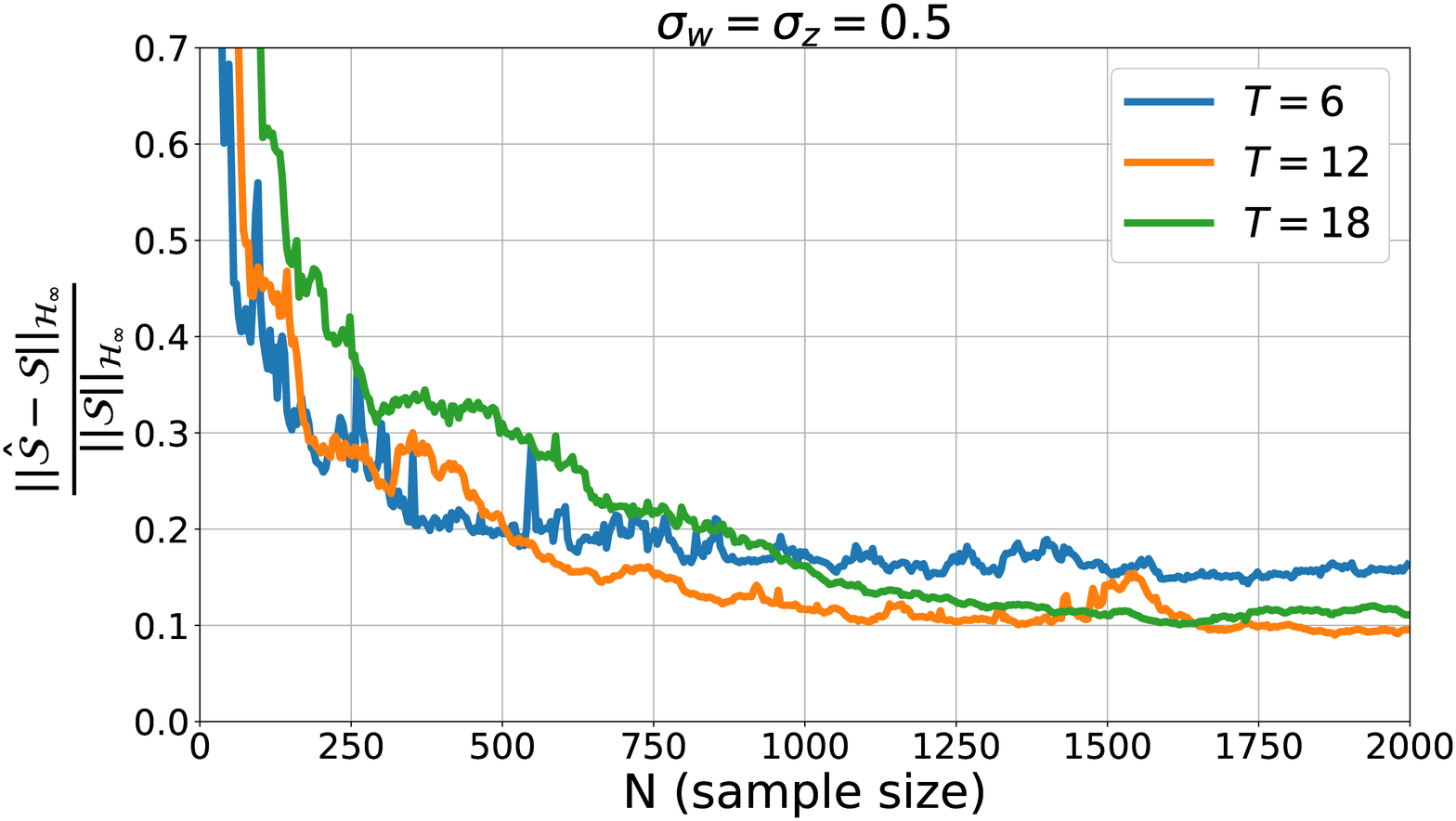}
        \caption{}
    \end{subfigure} ~
    \begin{subfigure}[b]{0.5\textwidth}
        \includegraphics[width=\textwidth]{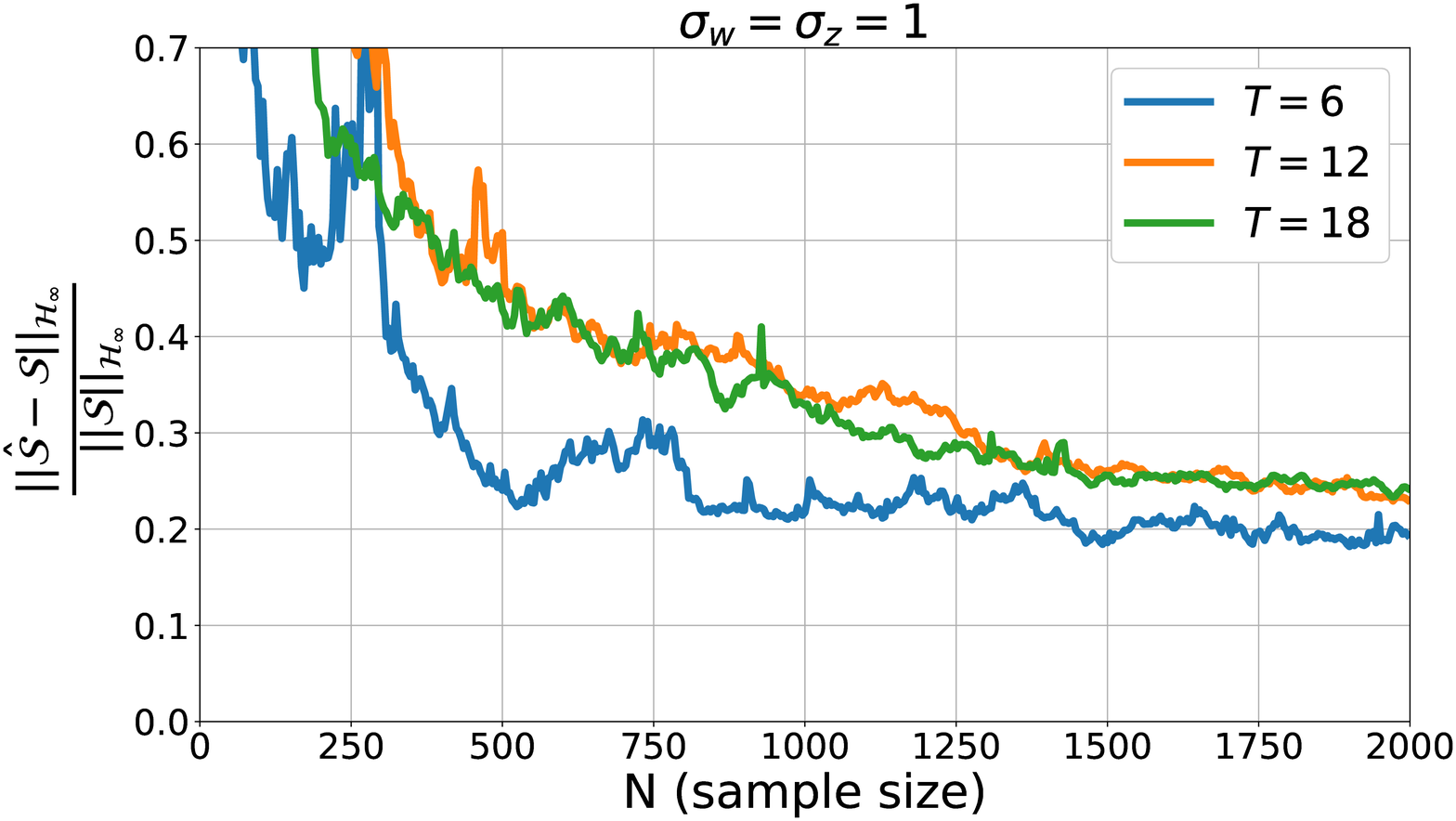}
        \caption{}
    \end{subfigure}
    \caption{\small Relative estimation errors for systems $\Sc,\hat{\Sc}$ for varying noise levels. $\hat{\Sc}$ is obtained by the \HK~procedure of Algorithm \ref{algo 1}. Based on Theorem \ref{kh stable}, we expect improved estimation accuracy for larger $N$ and smaller $\sigma_w,\sigma_z$; since the error in estimating the system matrices is directly controlled by the error in the Markov parameter matrix $\cba{}-\cbe{}$.}
        \label{fig abc}
\end{figure}

In Figure \ref{fig abc} we study the stability of the \HK~procedure which returns a realization up to a unitary transformation as described in Theorem \ref{kh stable}. Hence, rather than focusing on individual outputs $\Ah,\Bh,\Ch$ we directly study the LTI systems $\Sc=\text{LTI-sys}(\A,\B,\Cb,\Db)$ and $\hat{\Sc}=\text{LTI-sys}(\Ah,\Bh,\Ch,\hat{\Db})$. In particular, we focus on the ${\mathcal H}_\infty$ norm of the error $\Sc-\hat{\Sc}$. During this process, we clipped the singular values of $\Ah$ at $0.99$ i.e. if $\Ah$ has a singular value larger than $0.99$, we replace it by $0.99$ in the SVD of $\Ah$ which returns a new $\Ah$ whose singular vectors are same but singular values are clipped. This essentially corresponds to projecting the estimated system on the set of stable systems. While we verified that $\|\Ah\|>0.99$ rarely happens for large $N$, clipping ensures that ${\mathcal H}_\infty$ norm is always bounded and smooths out the results. Figure \ref{fig abc} illustrates the normalized ${\mathcal H}_\infty$ error $\frac{\|\hat{\Sc}-\Sc\|_{{\mathcal H}_\infty}}{\|\Sc\|_{{\mathcal H}_\infty}}$ for varying $\sigma_w=\sigma_z$ and $\TK\in\{6,12,18\}$. For zero-noise regime, $\TK=18$ outperforms the rest demonstrating the benefit of using a larger $T$ to overcome the contribution of the $\sigma_e$ term. In the other regimes, all $\TK$ choices perform fairly similar; however $\TK=6$ appears to suffer less from increasing noise levels $\sigma_w,\sigma_z$. Another observation is that for very small sample size $N$, $\TK=6$ converges faster than the others. This is supported by our Theorem \ref{main thm} as $\TK=6$ has less unknowns and the minimal $N$ is in the order of $\TK p$, hence smaller $\TK$ means faster estimation.

We remark that one might be interested in other metrics to assess the error such as Frobenius norm. While not shown in the figures, we also verified that the Frobenius norm $\tf{\cba{}-\cbe{}}$ (and the errors for $\Cb\B,\Db,\Hb$ as well as $\|\Sc-\hat{\Sc}\|_{{\mathcal H}_2}$) behaves in a similar fashion to spectral norm and ${\mathcal H}_\infty$ norm.

\section{Conclusions}
In this paper, we analyzed the sample complexity of linear system identification from input/output data. Our analysis neither requires multiple independent trajectories nor relies on splitting the trajectory into non-overlapping intervals, therefore makes very efficient use of the available data from a single trajectory. More crucially, it does not rely on state measurements and works with only the inputs and outputs. Based on this analysis, we showed that one can approximate system's Hankel operator using near optimal amount of samples and shed light on the stability of finding a balanced realization.

There are many directions for future work. First, we are interested in combining our results with control synthesis techniques based on Markov parameters. Second, it is shown empirically that minimizing the rank or nuclear norm of the estimated Hankel matrix as a denoising step (see e.g., \cite{fazel2013hankel}) works better than Ho-Kalman. It is of interest to analyze the stability of such optimization-based algorithms. Finally, it would be interesting to see what type of recovery guarantees can be obtained if additional constraints, such as subspace constraints, on the system matrices are known \cite{fattahi2018data}.

\section*{Acknowledgements}

S.O. would like to thank Mahdi Soltanolkotabi for pointing out Lemma $5.14$ of \cite{tu2015low} as well as Babak Hassibi and Jay Farrell for constructive feedback. Authors would like to thank Zhe Du for a careful reading of the manuscript and helpful suggestions. N.O. would like to thank Glen Chou for proofreading a draft and Dennis Bernstein for comments on the  \HK~algorithm. The work of N.O. is supported in part by DARPA grant N66001-14-1-4045 and ONR grant N000141812501.

\bibliographystyle{plain}
\bibliography{Bibfiles}

\newpage
\appendix

\section{Proof of the Results on Learning Markov Parameters}
We first describe the basic proof idea. Following equation \eqref{lsp eq}, to further simplify the notation, define the matrices
\begin{align}
&\W=[\wb_{\TK},~\wb_{\TK+1},~\dots,~\wb_{\bN}]^*\in\R^{N\times \TK n},\label{matrices}\\
&\Eb=[\e_{\TK},~\e_{\TK+1},~\dots,~\e_{\bN}]^*\in\R^{N\times n},\nn\\
&\Zb=[\z_{\TK},~\z_{\TK+1},~\dots,~\z_{\bN}]^*\in\R^{N\times m}.\nn
\end{align}
With these variables, we have the system of equations
\[
\Y=\Ub\cba{}^*+\Eb+\Zb+\W\ca{}^*.
\]
 Following \eqref{pinv est}, estimation error is given by
\begin{align}
(\cbe{}-\cba{})^*=(\Ub^*\Ub)^{-1}\Ub^*(\W\ca{}^*+\Zb+\Eb).\label{pinv terms}
\end{align}
The spectral norm of the error can be bounded as
\begin{align}
\|(\cbe{}-\cba{})^*\|\leq \|(\Ub^*\Ub)^{-1}\| (\|\Ub^*\W\|\|\ca{}^*\|+\|\Ub^*\Zb\|+\|\Ub^*\Eb\|).\label{error decomp}
\end{align}
Each of these terms will be bounded individually. The bounds on $\|(\Ub^*\Ub)^{-1}\| $ and $\|\Ub^*\W\|$ will be obtained by using the properties of random circulant matrices in Section \ref{circ sec}. $\|\Ub^*\Zb\|$ is arguably the simplest term due to $\Zb$ being an i.i.d.~Gaussian matrix. It is bounded via Lemma \ref{out noise}. Finally, $\|\Ub^*\Eb\|$ term will be addressed by employing a martingale based argument in Section \ref{sec martin}. We first prove Theorem \ref{main thm} which is our main theorem. It will be followed by the proof of Theorem \ref{main thm simple}.
\subsection{Proof of Theorem \ref{main thm}}\label{sec main prf}
\begin{proof} The proof is obtained by combining estimates from the subsequent sections. Set $\Theta=\log^2(2\TK p)\log^2(2Np)$ to simplify the notation. Picking $c\geq (\log2)^{-4}$, our assumption of $N\geq c\TK p\Theta$ implies $N\geq \TK$ and $N\geq (\bN+1)/2$ (using $\bN=N+\TK-1$). Consequently, $\log^2(2\bN p)\leq 4{\log^2(2Np)}$ and we have $N\geq (c/4)\TK p\log^2(2\TK p)\log^2(2\bN p)$. This fact will be useful when we need to utilize results of Section \ref{circ sec}. We first address the $\Zb$ component of the error which is rather trivial to bound.
\begin{lemma}\label{out noise} Let $\M\in\R^{m\times n}$ be a tall matrix ($m\geq n$) with $\|\M\|\leq \eta$. Let $\Gb\in\R^{m\times k}$ be a matrix with independent standard normal entries. Then, with probability at least $1-2\exp(-t^2/2)$,
\[
\|\M^*\Gb\|\leq\eta(\sqrt{2(n+k)}+t).
\]
In particular, setting $t=\sqrt{2(n+k)}$, we find $\|\M^*\Gb\|\leq \eta \sqrt{8(n+k)}$ with probability at least $1-2\exp(-(n+k))$.
\end{lemma}
\begin{proof} Suppose $\M$ have singular value decomposition $\M=\Vb_1\bSi\Vb_2^*$ where $\Vb_1\in\R^{m\times n}$. Observe that $\bar{\Gb}=\Vb_1^*\Gb\in\R^{n\times k}$ have i.i.d. $\Nn(0,1)$ entries. Also $\E[\|\bar{\Gb}\|]\leq \sqrt{n}+\sqrt{k}\leq \sqrt{2(n+k)}$. Applying Lipschitz Gaussian concentration on spectral norm, with probability at least $1-2\exp(-t^2/2)$,
\[
\|\M^*\Gb\|=\|\Vb_2\bSi\bar{\Gb}\|=\|\bSi\bar{\Gb}\|\leq \eta (\sqrt{2(n+k)}+t).
\]
\end{proof}
The following corollary states the estimation error due to measurement noise ($\Zb$ term).
\begin{corollary} \label{cor a2}Let $\Ub\in\R^{N\times \TK p}$ be the data matrix as in \eqref{label and data} and let $\Zb\in\R^{N\times m}$ be the measurement noise matrix from \eqref{matrices}. Suppose $N\geq c\TK p\Theta$ for some absolute constant $c>0$. With probability at least $1-2\exp(-(\TK p+m))-\exp(-\Theta)$,
\[
\|\Ub^*\Zb\|\leq 4\sigma_u\sigma_z\sqrt{N(\TK p+m)}.
\]
\end{corollary}
\begin{proof} Set $\eta=\sqrt{2N}\sigma_u$. Using $\bN\geq N$, Lemma \ref{lem cond} yields that 
\begin{align}
\Pro(\|\Ub\|\leq \eta)\geq 1-\exp(-\Theta).\label{u upp bound}
\end{align} Hence, combining Lemmas \ref{lem cond} and \ref{out noise}, using the fact that $\Zb,\Ub$ are independent, and adjusting for $\Zb$'s variance $\sigma_z$, we find the result. 
\end{proof}
Next, we apply Lemmas \ref{lem cond} and \ref{cross prod} to find that, for sufficiently large $c>0$, whenever $N\geq c\TK p\Theta$,
\begin{align}
\|(\Ub^*\Ub)^{-1}\|\leq 2\sigma_u^{-2}/N,~\|\Ub^*\W\|\leq \frac{1}{2}\sigma_u\sigma_w\max\{\sqrt{\NW N}, \NW\},\label{uw est}
\end{align}
where $\NW=c\TK q\log^2(2\TK q)\log^2(2Nq)$ and $q=p+n$ with probability at least $1-2\exp(-\Theta)$. Finally, applying Theorem \ref{ub eb prod} with $\gamma=\frac{\|\Ginf\|\Phi(\A)^2\|\Cb\A^{\TK-1}\|^2}{1-\rho(\A)^{2\TK}}$, with probability at least $1-\TK(\exp(-100\TK p)+2\exp(-100m))$,
\[
\|\Ub^*\Eb\|\leq c_3\sigma_u\sqrt{\TK \max\{N,\frac{m\TK}{1-\rho(\A)^{\TK}}\} \max\{\TK p,m\}  \gamma}.
\]
Combining all of the estimates above via union bound and substituting $\theta$, with probability at least,
\[
1-2\exp(-(\TK p+m))-3(2Np)^{-\log(2Np)\log^2(2\TK p)}-\TK(\exp(-100\TK p)+2\exp(-100m)),
\]
the error term $\|\cba{}-\cbe{}\|$ of \eqref{error decomp} is upper bounded by $\frac{R_z+R_e+R_w}{\sigma_u\sqrt{N}}$ where
\begin{align}
&R_z={8\sigma_z\sqrt{\TK p+m}},\\
&R_w={\sigma_w\|\ca{}\|\max\{\sqrt{\NW},\NW/\sqrt{N}\}},\\
&R_e={2C\sqrt{(1+\frac{m\TK}{N(1-\rho(\A)^{\TK})})(\TK p+m)\TK \gamma}}.
\end{align}
Absorbing the $\times 2$ multiplier of $R_e$ into $C$ and observing $\TK\gamma=\sigma_e^2$, we conclude with the desired result.
\end{proof}

\subsection{Proof of Theorem \ref{main thm simple}}
The proof uses the same strategy in Section \ref{sec main prf} with slight modifications. We will repeat the argument for the sake of completeness. First of all, we utilize the same estimates based on Lemmas \ref{lem cond} and \ref{cross prod}, namely \eqref{uw est} and \eqref{u upp bound} ($\|\Ub\|\leq \sqrt{2N}\sigma_u$) which hold with probability at least $1-3(2Np)^{-\log(2Np)\log^2(2\TK p)}$. Observe that $N\geq N_0\geq \NW=c\TK (p+n) \log^2(2\TK (p+n))\log^2(2N(p+n))$ hence, we have that
\[
\|\Ub^*\W\|\leq \frac{1}{2}\sigma_u\sigma_w{\sqrt{\NW N}}\leq \frac{1}{2}\sigma_u\sigma_w{\sqrt{N_0 N}}.
\]
We use Lemma \ref{out noise} with $t=\sqrt{2\TK q}$ to obtain $\Pro(\|\Ub^*\Zb\|\leq 4\sigma_u\sigma_z{\sqrt{\TK q N}})\geq 1-2\exp(-\TK q)$.

Finally, to bound the contribution of $\Eb$ we again apply Theorem \ref{ub eb prod}. Since $\rho(\A)^{\TK}\leq 0.99$, picking sufficiently large $c$, we observe that
\[
\max\{N,\frac{m\TK}{1-\rho(\A)^{\TK}}\}=N
\]
Hence, using $\sigma_e=\sqrt{\gamma\TK}$ and applying Theorem \ref{ub eb prod} yields that for some $C>0$
\[
\|\Ub^*\Eb\|\leq C{\sigma_u\sqrt{\TK N(\TK p+m)  \gamma}}\leq C{\sigma_u\sigma_e\sqrt{N\TK q}},
\]
holds with probability at least $1-\TK(\exp(-100\TK p)+2\exp(-100m))$. Union bounding over all these events and following \eqref{error decomp}, with probability at least,
\[
1-2\exp(-\TK q)-3(2Np)^{-\log(2Np)\log^2(2\TK p)}-\TK(\exp(-100\TK q)+2\exp(-100m)),
\]
we find the spectral norm estimation error of
\begin{align}
\|\cbe{}-\cba{}\|&\leq \frac{\frac{1}{2}\sigma_u\sigma_w\|\ca{}\|\sqrt{N_0 N}+4\sigma_u\sigma_z\sqrt{\TK q N}+C{\sigma_u\sigma_e\sqrt{\TK q N}}}{(\sigma^2_uN)/2}\\
&\leq \frac{\sigma_w\|\ca{}\|\sqrt{N_0 }+8\sigma_z\sqrt{\TK q }+2C{\sigma_e\sqrt{\TK q }}}{\sigma_u\sqrt{N}},
\end{align}
which is the desired bound after ensuring $\max\{8,2C\}^2\TK q\leq N_0$ by picking the constant $c$ (which leads $N_0$) to be sufficiently large.

\subsection{Proof of Theorem \ref{truncate thm}}
\begin{proof} Let us start with $\Gb^{(\infty)}$ estimate. First note that, the tail spectral norm is bounded via \eqref{tail spec}. Picking the proposed $\TK\geq \TK_1=1 -\frac{\log(2{\eps_0^{-1}(1-\rho(\A))^{-1}\Phi(\A) \|\Cb\|\|\B\|}\sqrt{\frac{N}{Tp+m}})}{\log(\rho(\A))}$ implies right hand side of \eqref{tail spec} can be upper bounded as
\begin{align}\label{tail bound}
 \frac{\Phi(\A) \|\Cb\|\|\B\|\rho(\A)^{\TK-1}}{1-\rho(\A)}\leq\frac{1}{2} \eps_0 \sqrt{\frac{Tp+m}{N}}\iff \rho(\A)^{-(\TK-1)}\geq  \frac{\Phi(\A) \|\Cb\|\|\B\|\sqrt{\frac{N}{Tp+m}}}{\eps_0(1-\rho(\A))}
\end{align}
Next, we will bound the spectral difference of order $\TK$ finite responses $\cba{}$ and $\cbe{}$. Let $\TK\geq -\frac{1}{\log(\rho(\A))}$ to ensure $\rho(\A)^\TK\leq 1/2$. Applying Theorem \ref{main thm}, we will show that individual error summands due to $R_w,R_e,R_z$ are upper bounded. First, Theorem \ref{main thm} is applicable due to the choice of $N$. $R_w$ summand is zero as $\sigma_w=0$. Second, observe that, for some $C>0$
\[
\frac{R_e}{\sigma_u\sqrt{N}}\leq {\frac{C}{4}\sigma_e\sqrt{(1+\frac{m\TK}{N(1-\rho(\A)^{\TK})})(\TK p+m)}}\leq \frac{\sigma_e}{\sigma_u}\frac{C}{4}\sqrt{1+\frac{m}{p}}\sqrt{\frac{Tp+m}{N}}.
\]
where we used $\sqrt{1+2mT/N}\leq \sqrt{1+\frac{m}{p}}$. Since $\sigma_w=0$, define,
\[
\Binf=\frac{\Ginf}{\sigma_u^2}=\sum_{i=0}^\infty\A^i\B\B^*(\A^*)^i.
\]
 Note that
\begin{align}
\frac{\sigma_e}{\sigma_u}&\leq 2\Phi(\A)\|\Cb\A^{\TK-1}\|\sqrt{\TK\|\Binf\|}\leq 2\Phi(\A)^2\|\Cb\|\rho(\A)^{\TK-1}\sqrt{\TK\|\Binf\|}\\
&\leq 2\sqrt{\frac{p}{p+m}}\eps_0/C,
\end{align}
which is guaranteed by $\TK\geq \TK_2= 1 -\frac{\log({C\eps_0^{-1}\Phi(\A)^2 \|\Cb\|\sqrt{T\|\Binf\|(1+m/p)}}{})}{\log(\rho(\A))}$ and ensures $\frac{R_e}{\sigma_u\sqrt{N}} \leq \frac{\eps_0}{2} \sqrt{\frac{Tp+m}{N}}$. Combining with $R_z$ bound of Theorem \ref{main thm} and tail bound of \eqref{tail bound}, these yield
\[
\|\Gb^{(\infty)}-\hat{\Gb}^{(\infty)}\|\leq (8\frac{\sigma_z}{\sigma_u}+\eps_0) \sqrt{\frac{Tp+m}{N}},
\]
whenever $N$ is stated as above and $T$ obeys $T\geq \max(-\frac{1}{\log(\rho(\A))},T_0)$ where 
\[
T_0:=\frac{c_0+\log(\Phi(\A)^2\|\Cb\| \eps_0^{-1})+\log(\max\{(1-\rho(\A))^{-1}\|\B\| \sqrt{\frac{N}{Tp+m}},\sqrt{T\|\Binf\|(1+m/p)}\})}{-\log(\rho(\A))}\geq \max(T_1,T_2).
\]
Treating $\A,\B,\Cb$ related variables in the numerator as constant terms (which are less insightful than the $\log(\rho(\A))$ term for our purposes), we find the condition \eqref{tk cond}.

To proceed, we  wish to show the result on Hankel matrices $\Hb^{(\infty)}$ and $\hat{\Hb}^{(\infty)}$. We shall decompose the $\Hb^{(\infty)}$ matrix as $\Hb^{(\infty)}=\Hb_{\text{main}}+\Hb_{\text{tail}}$ (same for $\hat{\Hb}$). $\Hb_{\text{main}},\hat{\Hb}_{\text{main}}$ are the $m\times p$ blocks corresponding to the first $\TK$ Markov parameters and their estimates. Observe that $\Hb_{\text{main}}$ lives on the upper-left $\TK\times \TK$ submatrix. Furthermore, the set of non-zero blocks in each of its first $\TK$ block-rows of size $m\times \TK p$ is a submatrix of $\cba{}$. For instance in \eqref{hmatrices}, non-zero rows of $\hat{\Hb}$ are all submatrices of $\cbe{}$. Consequently, adding spectral norms of nonzero rows and using the above bound on $\cba{}$ estimate, we have that
\[
\|\hat{\Hb}_{\text{main}}-{\Hb}_{\text{main}}\|\leq \TK \|\cba{}-\cbe{}\|\leq \TK(8\frac{\sigma_z}{\sigma_u}+\frac{\eps_0}{2}) \sqrt{\frac{Tp+m}{N}},
\]
where $\eps_0/2$ instead of $\eps_0$ is due to lack of tail terms. What remains is the $\Hb_{\text{tail}}$ term. Note that $\hat{\Hb}_{\text{tail}}=0$. $\Hb_{\text{tail}}$ matrix is composed of anti-diagonal blocks that start from $\TK+1$ till infinity. The non-zero blocks of $i$th anti-diagonal ($i\geq \TK+1$) are all equal to $\Cb\A^{i-2}\B$ due to Hankel structure, hence its spectral norm is equal to $\|\Cb\A^{i-2}\B\|$. Consequently, the spectral norm of $\Hb_{\text{tail}}$ can be obtained by adding the spectral norm of non-zero anti-diagonal matrices which is given by \eqref{tail spec} and is upper bounded by $\eps_0/2$ in \eqref{tail bound}. Hence, 
\[
\|\hat{\Hb}^{(\infty)}-{\Hb}^{(\infty)}\|\leq \|\hat{\Hb}_{\text{main}}-{\Hb}_{\text{main}}\|+\|\hat{\Hb}_{\text{tail}}-{\Hb}_{\text{tail}}\|\leq \TK(8\frac{\sigma_z}{\sigma_u}+{\eps_0}{}) \sqrt{\frac{Tp+m}{N}},
\]
concluding the proof.
\end{proof}

\section{Proof of the Ho-Kalman Stability}
In this section, we provide a proof for the stability of the \HK~procedure. Since system is assumed to be observable and controllable and $\TK_1,\TK_2$ are assumed to be sufficiently large, $\text{rank}(\Lb)=n$ throughout this section. Recall that, given Markov parameter matrices $\cba{},\cbe{}$, the matrices $\Hb,\Hbm,\Hbl,\Hbp$ (with $\Hbl=\Hbm$ as $\Hbm$ is rank $n$) correspond to $\cba{}$ and the matrices $\hat{\Hb},\Hbhm,\Hbh,\Hbhp$ correspond to $\cbe{}$. 
We will show that \HK~state-space realizations corresponding to $\cba{}$ and $\cbe{}$ are close to each other as a function of $\|\cba{}-\cbe{}\|$. We first provide a proof of Lemma \ref{g to h and l}.

\subsection{Proof of Lemma \ref{g to h and l}}
We wish to show that $\Hb-\hat{\Hb}$ and $\Hbl-{\Hbh}$ can be upper bounded in terms of $\cba{}-\cbe{}$ via \eqref{cba upp}.
$\Hb^--\hat{\Hb}^-$ is a submatrix of $\Hb-\hat{\Hb}$ hence we have 
\[
\|\Hb^--\hat{\Hb}^-\|\leq \|\Hb-\hat{\Hb}\|.
\]
Denote the $i$th block row of $\Hb$ by $\Hb[i]$. Since $\Hb[i]$ (for all $i$) is a submatrix of the Markov parameter matrix $\cba{}$, we have that $\|\Hb[i]-\hat{\Hb}[i]\|\leq \|\cba{}-\cbe{}\|$. Hence, the overall matrix $\Hb$ satisfies
\[
\|\Hb-\hat{\Hb}\|=\left\|\begin{bmatrix}\Hb[1]-\hat{\Hb}[1]\\\vdots\\\Hb[T_1]-\hat{\Hb}[T_1]\end{bmatrix}\right\|\leq \sqrt{\TK_1}\max_{1\leq i\leq T_1}\|\Hb[i]-\hat{\Hb}[i]\|\leq \sqrt{\TK_1} \|\cba{}-\cbe{}\|.
\]
Similarly, columns of $\Hb$ are also submatrices of $\cba{}$. Repeating same argument for columns, yields
\[
\|\Hb-\hat{\Hb}\|\leq \sqrt{\TK_2+1} \|\cba{}-\cbe{}\|.
\]
Combining both, we find \eqref{cba upp}. The bound \eqref{cba upp2} is based on singular value perturbation. First, noticing that rows/columns of $\Hbm$ are again copied from $\cba{}$ and carrying out the same argument, we have that
\[
\|\Hbm-\Hbhm\|\leq\sqrt{\min\{\TK_1,\TK_2\}}\|\cba{}-\cbe{}\|.
\] 
Recall that $\Hbl=\Hbm$ and $\Hbh$ is the rank-$n$ approximations of $\Hbhm$. Denoting $i$th singular value of $\Hbhm$ by $\sigma_i(\Hbhm)$, standard singular value perturbation bound yields 
\[
\sigma_{n+1}(\Hbhm)=\|\Hbhm-\Hbh\|\leq \|\Hbhm-\Hbm\|.
\] Consequently, using $\Hbl=\Hb^-$,
\begin{align}
\|\Hbl-{\Hbh}\|\leq\|\Hbm-\Hbhm\|+\|\Hbhm-\Hbh\|\leq 2\|\Hbm-\hat{\Hb}^-\|\leq 2\sqrt{\min\{\TK_1,\TK_2\}}\|\cba{}-\cbe{}\|.\nn
\end{align}
\subsection{Robustness of Singular Value Decomposition}
The next theorem shows robustness of singular value decompositions of $\Lb$ and $\hat{\Lb}$ in terms of $\|\Lb-\hat{\Lb}\|$. It is obtained by using Lemma $5.14$ of \cite{tu2015low} and provides simultaneous control over left and right singular vector subspaces. This is essentially similar to results of Wedin and Davis-Kahan \cite{dopico2000note,yu2014useful} with the added advantage of simultaneous control which we crucially need for our result.
\begin{lemma}\label{sin theta sim} Suppose $\smin\geq 2\|\Hbl-\Hbh\|$ where $\smin$ is the smallest nonzero singular value (i.e.~$n$th largest singular value) of $\Lb$. Let rank $n$ matrices $\Hbl$, $\Hbh$ have singular value decompositions $\Ub\bSi\V^*$ and $\hat{\Ub}\hat{\bSi}\hat{\V}^*$. There exists an $n\times n$ unitary matrix $\Tb$ so that
\[
\tf{\Ub\bSi^{1/2}-\hat{\Ub}\hat{\bSi}^{1/2}\Tb}^2+\tf{\V\bSi^{1/2}-\hat{\V}\hat{\bSi}^{1/2}\Tb}^2\leq 5n\|\Lb-\hat{\Lb}\|.
\]
\end{lemma}
\begin{proof} Direct application of Theorem $5.14$ of \cite{tu2015low} guarantees the existence of a unitary $\Tb$ such that
\[
\text{LHS}=\tf{\Ub\bSi^{1/2}-\hat{\Ub}\hat{\bSi}^{1/2}\Tb}^2+\tf{\V\bSi^{1/2}-\hat{\V}\hat{\bSi}^{1/2}\Tb}^2\leq \frac{2}{\sqrt{2}-1}\frac{\tf{\Lb-\hat{\Lb}}^2}{\sigma_{\min}(\Lb)}.
\]
To proceed, using $\text{rank}(\Lb-\hat{\Lb})\leq 2n$ and $\smin\geq 2\|\Hbl-\Hbh\|\geq \sqrt{2/n}\tf{\Hbl-\Hbh}$, we find
\[
\text{LHS}\leq \frac{\sqrt{2n}}{\sqrt{2}-1}{\tf{\Lb-\hat{\Lb}}}\leq \frac{2{n}}{\sqrt{2}-1}{\|\Lb-\hat{\Lb}\|}\leq 5n\|\Lb-\hat{\Lb}\|.
\]
\end{proof}
Observe that our control over the subspace deviation improves as the perturbation $\|\Lb-\hat{\Lb}\|$ gets smaller. The next lemma is a standard result on singular value deviation.
\begin{lemma} \label{smin dev} Suppose $\smin\geq 2\|\Hbl-\Hbh\|$. Then, $\|\Hbh\|\leq 2\|\Hbl\|$ and $\sigma_{\min}(\Hbh)\geq \smin/2$.
\end{lemma}

Using these, we will prove the robustness of \HK. The robustness will be up to a unitary transformation similar to Lemma \ref{sin theta sim}.

\subsection{Proof of Theorem \ref{kh stable}}
\begin{proof} Consider the SVD of $\Hbl$ given by $\Ub\bSi\V$ and SVD of $\Hbh$ given by $\hat{\Ub}\hat{\bSi}\hat{\V}$ where $\bSi,\hat{\bSi}\in\R^{n\times n}$ (recall that $\text{rank}(\Hbl)=n$ since we assumed system is observable and controllable). Define the observability/controllability matrices ($\Ob=\Ub\bSi^{1/2},\Qb=\bSi^{1/2}\V$) associated to $\Hb$ and ($\hat{\Ob}=\hat{\Ub}\hat{\bSi}^{1/2},\hat{\Qb}=\hat{\bSi}^{1/2}\hat{\Vb}$) associated to $\hat{\Hb}$. Lemma \ref{sin theta sim} automatically gives control over these as it states the existence of a unitary matrix $\Tb$ such that
\[
\tf{\Ob-\hat{\Ob}\Tb}^2+\tf{\Qb-\Tb^*\hat{\Qb}}^2\leq 5n\|\Lb-\hat{\Lb}\|.
\]
Since $\Cbb$ is a submatrix of $\Ob$ and $\Bb$ is a submatrix of $\Qb$, we immediately have the same upper bound on $(\Cbb,\Ch)$ and $(\Bb,\Bh)$ pairs.

The remaining task is to show that $\Ah$ and $\Ab$ are close. Let $\X=\hat{\Ob}\Tb$, $\Y=\Tb^*\hat{\Qb}$. Now, note that
\begin{align}
\tf{\Ab-\Tb^*\Ah\Tb}=\tf{\Ob^\dagger \Hb^+\Qb^\dagger-\Tb^*\hat{\Ob}^\dagger \hat{\Hb}^+\hat{\Qb}^\dagger\Tb}=\tf{\Ob^\dagger \Hb^+\Qb^\dagger-\X^{\dagger} \hat{\Hb}^+\Y^{\dagger}}.
\end{align}
Consequently, we can decompose the right hand side as
\begin{align}
\tf{\Ob^\dagger \Hb^+\Qb^\dagger-\X^{\dagger} \hat{\Hb}^+\Y^{\dagger}}\leq &\tf{(\Ob^\dagger-\X^\dagger) \Hb^+\Qb^\dagger}+\tf{\X^\dagger (\Hb^+-\hat{\Hb}^+)\Qb^\dagger}\label{rhs hk}\\
&+\tf{\X^\dagger \hat{\Hb}^+(\Qb^\dagger-\Y^\dagger)}.\nn
\end{align}
We treat the terms on the right hand side individually. First, pseudo-inverse satisfies the perturbation bound \cite{meng2010optimal,wedin1973perturbation}
\[
\tf{\Ob^\dagger-\X^\dagger}\leq \tf{\Ob-\X}\max\{\|\X^\dagger\|^2,\|{\Ob}^\dagger\|^2\}\leq \sqrt{5n\|\Lb-\hat{\Lb}\|}\max\{\|\X^\dagger\|^2,\|{\Ob}^\dagger\|^2\}.
\]
We need to bound the right hand side. Luckily, Lemma \ref{smin dev} trivially yields the control over the top singular values of pseudo-inverses namely 
\[
\max\{\|\X^\dagger\|^2,\|{\Ob}^\dagger\|^2\}=\max\{\frac{1}{\smin},\frac{1}{\sigma_{\min}(\Hbh)}\}\leq \frac{2}{\smin}.
\] Combining the last two bounds, we find
\[
\tf{\Ob^\dagger-\X^\dagger}\leq \frac{2\sqrt{5n\|\Lb-\hat{\Lb}\|}}{\smin}
\]
The identical bounds hold for $\Qb,\Y$. For the second term on the right hand side of \eqref{rhs hk}, we shall use the estimate
\[
\|\X^\dagger (\Hb^+-\hat{\Hb}^+)\Qb^\dagger\|_F\leq \sqrt{n}\|\X^\dagger (\Hb^+-\hat{\Hb}^+)\Qb^\dagger\|\leq \frac{2\sqrt{n}}{\smin}\|\Hb^+-\hat{\Hb}^+\|.
\]
Finally, we will use the standard triangle inequality to address the $\hat{\Hb}^+$ term: $\|\hat{\Hb}^+\|\leq \|{\Hb}^+\|+\|\Hb^+-\hat{\Hb}^+\|$. Combining all of these, we obtain the following bounds
\begin{align}
\tf{(\Ob^\dagger-\X^\dagger) \Hb^+\Qb^\dagger}&\leq \tf{\Ob^\dagger-\X^\dagger}\| \Hb^+\|\|\Qb^\dagger\|\\
&\leq \frac{\sqrt{20n\|\Lb-\hat{\Lb}\|}}{\smin}\sqrt{\frac{2}{\smin}}\|\Hb^+\|\\
&\leq \frac{7\sqrt{n\|\Lb-\hat{\Lb}\|}}{\smin^{3/2}}\|\Hb^+\|
\end{align}
\begin{align}
\tf{\X^\dagger \hat{\Hb}^+(\Qb^\dagger-\Y^\dagger)}&\leq \|\X^\dagger\|\|\hat{\Hb}^+\|\tf{\Qb^\dagger-\Y^\dagger}\\
&\leq\frac{7\sqrt{n\|\Lb-\hat{\Lb}\|}}{\smin^{3/2}}(\|{\Hb}^+\|+\|\Hb^+-\hat{\Hb}^+\|)
\end{align}
\[
\|\X^\dagger (\Hb^+-\hat{\Hb}^+)\Qb^\dagger\|_F\leq \frac{2\sqrt{n}\|\Hb^+-\hat{\Hb}^+\|}{\smin}.
\]
Combining these three individual bounds and substituting in \eqref{rhs hk}, we find the overall bound 
\[
\tf{\Ab-\Tb^*\Ah\Tb}\leq \frac{14\sqrt{n}}{\smin}(\sqrt{\frac{\|\Lb-\hat{\Lb}\|}{\smin}}(\|{\Hb}^+\|+\|\Hb^+-\hat{\Hb}^+\|)+\|\Hb^+-\hat{\Hb}^+\|).
\]
\end{proof}

\subsection{Proof of Corollary \ref{cor hk}}

Using Lemma \ref{kh stable}, the condition $\|\Hb-\hat{\Hb}\|\leq \smin/4$ implies the condition \eqref{perturb req}. Consequently, inequalities \eqref{pair bounds} and \eqref{a bound} of Theorem \ref{kh stable} holds. \eqref{pair bounds 2} follows by using $\|\Hbl-\Hbh\|\leq 2\|\Hb-\hat{\Hb}\|$. The result on $\Ab$ is slightly more intricate. First, since $\Hb^+$ is a submatrix of $\Hb$
\[
\|\Hb^+-\hat{\Hb}^+\|\leq \|\Hb-\hat{\Hb}\|,~\|\Hb^+\|\leq \|\Hb\|
\]
Combining this with \eqref{cba upp2}, the right hand side of \eqref{a bound} can be upper bounded by
\[
\text{RHS}=\frac{14\sqrt{2}\sqrt{n}}{\smin}(\sqrt{\frac{\|\Hb-\hat{\Hb}\|}{\smin}}(\|{\Hb}\|+\|\Hb-\hat{\Hb}\|)+\|\Hb-\hat{\Hb}\|).
\]
Next, $\Lb=\Hb^-$ hence $4\|\Hb-\hat{\Hb}\|\leq \smin\leq \|\Lb\|\leq \|\Hb\|$. Hence $\|{\Hb}\|+\|\Hb-\hat{\Hb}\|\leq (5/4)\|{\Hb}\|$. Finally, 
\[
\frac{\|\Hb\|}{\sqrt{\smin}}\geq \sqrt{\|\Hb\|}\geq 2\sqrt{\|\Hb-\hat{\Hb}\|}.
\]
Combining the last two observations, $\text{RHS}$ can be upper bounded as
\begin{align}
\text{RHS}&=\frac{14\sqrt{2}\sqrt{n\|\Hb-\hat{\Hb}\|}}{\smin}(\frac{\|{\Hb}\|+\|\Hb-\hat{\Hb}\|}{\sqrt{\smin}}+\sqrt{\|\Hb-\hat{\Hb}\|})\\
&\leq \frac{14\sqrt{2}\sqrt{n\|\Hb-\hat{\Hb}\|}}{\smin}(\frac{5}{4}\frac{\|{\Hb}\|}{\sqrt{\smin}}+\sqrt{\|\Hb-\hat{\Hb}\|})\\
&\leq \frac{14\sqrt{2}\sqrt{n\|\Hb-\hat{\Hb}\|}}{\smin}(\frac{5}{4}\frac{\|{\Hb}\|}{\sqrt{\smin}}+\frac{1}{2}\frac{\|{\Hb}\|}{\sqrt{\smin}})\\
&\leq \frac{(7/4)14\sqrt{2}\sqrt{n\|\Hb-\hat{\Hb}\|}}{\smin}\frac{\|{\Hb}\|}{\sqrt{\smin}},
\end{align}
which is the advertised bound after noticing $(7/4)14\sqrt{2}\leq 50$.

\section{Restricted Isometry of Partial Circulant Matrices}\label{circ sec}
To proceed, let us describe the goal of this section. First, we would like to show that $\Ub\in\R^{N\times \TK p}$ is well conditioned when $N\gtrsim\order{\TK p}$ to ensure least-squares is robust. Next, we would like to have an accurate upper bound on the spectral norm of $\Ub^*\W$ to control the impact of noise $\w_t$. In particular, we will show that
\[
\|\Ub^*\W\|\lesssim \sigma_u\sigma_w\sqrt{N\TK(p+n)}.
\]
Both of these goals will be achieved by embedding $\Ub$ and $\W$ into proper circulant matrices. The same argument will apply to both scenarios. The key technical tool in our analysis will be the results of Krahmer et al. \cite{krahmer2014suprema} on restricted isometries of random circulant matrices.

The following theorem is a restatement of Theorem $4.1$ of Krahmer et al \cite{krahmer2014suprema}. We added a minor modification to account for the regime {\em{restricted isometry constant}} is greater than $1$. This result is proven in Section \ref{sec circ thm}. This theorem shows that arbitrary submatrices of random circulant matrices are well conditioned. It will play a crucial role in establishing the joint relation of the data matrix $\Ub$ and noise matrix $\W$. Main result of \cite{krahmer2014suprema} characterizes a uniform bound on all submatrices; however we only need a single submatrix for our results. {Hence, some of the logarithmic factors below might actually be redundant for the bound we are seeking.} 
\begin{theorem} \label{circ thm}Let $\Cb\in\R^{d\times d}$ be a circulant matrix where the first row is distributed as $\Nn(0,\Iden_d)$. Given $s\geq 1$, set $m_0=c_0s\log^2(2s)\log^2(2d)$ for some absolute constant $c_0>0$. Pick an $m\times s$ submatrix $\Sb$ of $\Cb$. With probability at least $1-(2d)^{-\log(2d)\log^2(2s)}$, $\Sb$ satisfies
\[
\|\frac{1}{m}\Sb^*\Sb-\Iden\|\leq \max\{\sqrt{\frac{m_0}{m}},\frac{m_0}{m}\}.
\]
\end{theorem}
The next two sections address the minimum singular value of the $\Ub$ matrix and upper bounding the maximum singular value of the $\Ub^*\W$ matrix by utilizing Theorem \ref{circ thm}.
\subsection{Conditioning of the Data Matrix}
\begin{lemma} \label{lem cond}Let $\Ub\in\R^{N\times \TK p}$ be the input data matrix as described in Section \ref{lsp sec}. Suppose the sample size obeys $N\geq c\TK p \log^2(2\TK p)\log^2(2\bN p)$ for sufficiently large constant $c>0$. Then, with probability at least $1-(2\bN p)^{-\log^2(2\TK p)\log(2\bN p)}$,
\[
2N\sigma_u^2\succeq\Ub^*\Ub\succeq  N{\sigma_u^{2}/2}.
\]
\end{lemma}
\begin{proof} The proof will be accomplished by embedding $\Ub$ inside a proper circulant matrix. Let $r(\vb):\R^d\rightarrow\R^d$ be the circulant shift operator which maps a vector $\vb\in\R^d$ to its single entry circular rotation to the right i.e. $r(\vb)=[\vb_d~\vb_1~\dots~\vb_{d-1}]\in\R^d$. Let $\Cb\in\R^{\bN p\times \bN p}$ be a circulant matrix where the first row (transposed) is given by
\[
\cb_1=[\ub_{\bN p}^*~\ub_{\bN p-1}^*~\dots~\ub_2^*~\ub_1^*]^*.
\]
The $i$th row of $\Cb$ is $\cb_i=r^{i-1}(\cb_1)$ for $1\leq i\leq \bN p$. Observe that $\Cb$ is a circulant matrix by construction. For instance all of its diagonal entries are equal to $\ub_{\bN p,1}$. {Additionally, note that second row of $\Cb$ starts with the last entry of $\ub_1$ hence entries of $\ub_i$ do not necessarily lie next to each other.} Focusing on the rightmost $\TK p$ columns, let $\Rbo{\TK p}$ be the operator that returns rightmost $\TK p$ entries of a vector. Our first observation is that 
\[
\Rbo{\TK p}(\cb_{1})=\ubb_{\TK}=[\ub_{\TK }^*~\ub_{\TK -1}^*~\dots~\ub_2^*~\ub_1^*]^*.
\]
Secondly, observe that for each $0\leq i\leq N-1$
\[
\Rbo{\TK p}(\cb_{1+ip})=\Rbo{\TK p}(r^{ip}(\cb_{1}))=[\ub_{\TK +i}^*~\ub_{\TK -1+i}^*~\dots~\ub_{2+i}^*~\ub_{1+i}^*]^*=\ubb_{\TK+i}.
\]
This implies that $\ubb_{\TK+i}$ is embedded inside right-most $\TK p$ columns and $1+ip$'th row of $\Cb$. Similarly, the input data matrix $\Ub\in\R^{N\times \TK p}$ is a submatrix of $\Cb$ with column indices $(\bN-\TK) p+1$ to $\bN p$ and row indices $1+ ip$ for $0\leq i\leq N-1$. Applying Theorem \ref{circ thm}, setting $N_0=c\TK p \log^2(2\TK p)\log^2(2\bN p)$, and adjusting for variance $\sigma_u^2$, with probability at least $1-(2\bN p)^{-\log^2(2\TK p)\log(2\bN p)}$, we have
\[
 {2\sigma_u^2}\Iden\succeq N^{-1}\Ub^*\Ub\succeq \frac{\sigma_u^2}{2}\Iden\implies 2N\sigma_u^2\succeq\Ub^*\Ub\succeq N {\sigma_u^{2}/2},
\]
whenever $N\geq N_0$.
\end{proof}
\subsection{Upper Bounding the Contribution of the Process Noise}
\begin{lemma}\label{cross prod} Recall $\Ub,\W$ from \eqref{label and data} and \eqref{matrices} respectively. Let $q=p+n$ and $N_0=c\TK q\log^2(2\TK q)\log^2(2\bN q)$ where $c>0$ is an absolute constant. With probability at least $1-(2\bN q)^{-\log^2(2 \TK q)\log(2\bN q)}$, 
\[
\|\Ub^*\W\|\leq \sigma_w\sigma_u\max\{\sqrt{{N_0N}},{N_0}\}.
\]
\end{lemma}
\begin{proof} The proof is identical to that of Lemma \ref{lem cond}. Set $q=p+n$. First, we define $\m_t=[\sigma_u^{-1}\ub_t^*~\sigma_w^{-1}\w_t^*]^*\in\R^{q}$ and $\bar{\m}_i=[\m_{i}^*,~\m_{i-1}^*,~\dots~\m_{i-\TK+1}^*]^*\in\R^{\TK q}$. We also define the matrix $\M=[\bar{\m}_{\TK}~\dots~\bar{\m}_{\TK+N-1}]^*\in\R^{N\times \TK q}$. Observe that by construction, $\sigma_u^{-1}\Ub,\sigma_w^{-1}\W$ are submatrices of $\M$. In particular, $(\sigma_u\sigma_w)^{-1}\Ub^*\W$ is an off-diagonal submatrix of $\M^*\M$ of size $\TK p\times \TK n$. This is due to the facts that i) $\sigma_u^{-1}\Ub$ is a submatrix of $\M$ characterized by the column indices 
\[
\{(i-1)q+j\bgl 1\leq i\leq \TK,~1\leq j\leq p\},
\]
 and ii) $\sigma_w^{-1}\W$ lies at the complementary columns. Observe that the spectral norm of $(\sigma_u\sigma_w)^{-1}\Ub^*\W$ can be upper bounded as
\begin{align}
(\sigma_u\sigma_w)^{-1}\|\Ub^*\W\|\leq \|\M^*\M-N\Iden\|.\label{uw bound}
\end{align}
\begin{proof} Since $(\sigma_u\sigma_w)^{-1}\Ub^*\W$ is an off-diagonal submatrix of $\M^*\M$, it is also a submatrix of $\M^*\M-\Iden$. Spectral norm of a submatrix is upper bounded by the norm of the original matrix hence the claim follows.
\end{proof}
In a similar fashion to Lemma \ref{lem cond}, we complete $\M$ to be a full circulant matrix as follows. Let $r(\vb):\R^d\rightarrow\R^d$ be the circulant shift operator as previously. Let $\Cb\in\R^{\bN q\times \bN q}$ be a circulant matrix with first row given by
\[
\cb_1=[\m_{\bN q}^*~\m_{\bN q-1}^*~\dots~\m_2^*~\m_1^*]^*.
\]
The $i$th row of $\Cb$ is $\cb_i=r^{i-1}(\cb_1)$ for $1\leq i\leq \bN q$. Let $\Rbo{\TK q}$ be the operator that returns rightmost $\TK q$ entries of a vector. Our first observation is that 
\[
\Rbo{\TK q}(\cb_{1})=\bar{\m}_\TK=[\m_{\TK }^*~\m_{\TK -1}^*~\dots~\m_2^*~\m_1^*]^*.
\]
Secondly, observe that for each $0\leq i\leq N-1$
\[
\Rbo{\TK q}(\cb_{1+iq})=\Rbo{\TK q}(r^{iq}(\cb_{1}))=[\m_{\TK +i}^*~\m_{\TK -1+i}^*~\dots~\m_{2+i}^*~\m_{1+i}^*]^*=\bar{\m}_{\TK+i}.
\]
This implies that $\bar{\m}_i$'s are embedded inside the rows of $\Rbo{\TK q}(\Cb)$ in an equally spaced manner with spacing $q$ for $\TK\leq i\leq \TK+N-1=\bN$. Hence, $\M$ is a $N\times \TK q$ submatrix of $\Cb$ where the column indices are the last $\TK q$ columns and the row indices are $1,1+q,\dots,1+(N-1)q$.

With this observation, we are ready to apply Theorem \ref{circ thm}. Theorem \ref{circ thm} states that for 
\[
N_0=c\TK q\log^2(2\TK q)\log^2(2\bN q),
\] with probability at least $1-(2\bN q)^{-\log^2(2\TK q)\log(2\bN q)}$, 
\[
\|\frac{1}{N}\M^*\M-\Iden\|\leq \max\{\sqrt{\frac{N_0}{N}},{\frac{N_0}{N}}\},
\]
which in turn implies $\|\Ub^*\W\|\leq \sigma_w\sigma_u\max\{\sqrt{{N_0N}},{N_0}\}$ via inequality \eqref{uw bound}.
\end{proof}
\section{Bounding the Error due to the Unknown State}\label{sec martin}
The goal of this section is bounding the estimation error due to the $\eb_t=\Cb\A^{\TK-1}\x_{t-\TK+1}$ term. As described in Section \ref{lsp sec} and \eqref{matrices}, we form the matrices $\Eb=[\eb_{\TK}~\dots~\eb_{\bN}]^*$ and $\Ub=[\ubb_{\TK}~\dots~\ubb_{\bN}]^*$. Our interest in this section is bounding $\|\Ub^*\Eb\|$. This term captures the impact of approximating the system with a finite impulse response of length $\TK$. We will show that
\[
\|\Ub^*\Eb\|\lesssim\sigma_u \sqrt{(\TK p+m)N\TK \|\Ginf\|\|\Cb\A^{\TK-1}\|^2}.
\]

The main challenge in analyzing $\Ub^*\Eb$ is the fact that $\{\e_t\}_{t=\TK}^{\bN}$ terms and $\{\ubb_t\}_{t=\TK}^{\bN}$ terms are dependent. In fact $\e_t$ contains a $\ub_{\tau}$ component inside for any $\tau\leq t-\TK$. The following theorem is our main result on bounding this term which carefully addresses these dependencies.
\begin{theorem} \label{ub eb prod} Suppose we are given $\Ub,\Eb$, as described in Section \ref{lsp sec} and \eqref{matrices}. Define $\gamma=\frac{\|\Ginf\|\Phi(\A)^2\|\Cb\A^{\TK-1}\|^2}{1-\rho(\A)^{2\TK}}$ and suppose $N\geq \TK$. Then, with probability at least $1-\TK(\exp(-100\TK p)+2\exp(-100m))$, 
\[
\|\Ub^*\Eb\|\leq c\sigma_u\sqrt{\TK \max\{N,\frac{m\TK}{1-\rho(\A)^{\TK}}\} \max\{\TK p,m\}  \gamma}.
\]
\end{theorem}
\begin{proof}
We first decompose $\Ub^*\Eb=\sum_{t=\TK}^{\bN} \ubb_t\eb_t^*$ into sum of $\TK$ smaller products. Given $0\leq t<\TK$, create sequences $S_t=\{t+\TK,t+2\TK,\dots,t+N_t\TK\}$ where $N_t$ is the largest integer satisfying $t+N_t\TK\leq \bN$. Each sequence has length $N_t$ which is at least $\lfloor N/\TK\rfloor $ and at most $\lfloor N/\TK\rfloor +1$. With this, we form the matrices 
\begin{align}
\Ub_t=[ \ubb_{t+\TK},~ \ubb_{t+2\TK},~\dots,~\ubb_{t+N_t\TK}]^*,~\Eb_t=[ \eb_{t+\TK},~ \eb_{t+2\TK},~\dots,~\eb_{t+N_t\TK}]^*.\label{decoupled}
\end{align}
Then, $\Ub^*\Eb$ can be decomposed as
\begin{align}
\Ub^*\Eb=\sum_{t=0}^{\TK-1}\Ub_t^*\Eb_t\implies \|\Ub^*\Eb\|\leq \sum_{t=0}^{\TK-1}\|\Ub_t^*\Eb_t\|.\label{sum spect}
\end{align}
Corollary \ref{simplified} provides a probabilistic spectral norm bound on each term of this decomposition on the right hand side. In particular, applying Corollary \ref{simplified}, substituting $\upsilon$ definition, and union bounding over $\TK$ terms, for all $t$, we obtain
\[
\|\Ub_t^*\Eb_t\|\leq c\sigma_u\sqrt{\max\{N,\frac{m\TK}{1-\rho(\A)^{\TK}}\} \max\{p,m/\TK\}  \gamma},
\]
with probability at least $1-\TK(\exp(-\TK q)+2\exp(-100m))$. This gives the advertised bound on $\Ub^*\Eb$ via \eqref{sum spect}.
\end{proof}

\subsection{Upper Bounding the Components of the Unknown State Decomposition}
Our goal in this section is providing an upper bound on the spectral norm of $\Ub_t^*\Eb_t$ which is described in \eqref{decoupled}. The following lemma provides a bound that decays with $1/\sqrt{N_t}$. The main tools in our analysis are the probabilistic upper bound on the $\Eb_t$ matrix developed in Section \ref{error mat sec} and martingale concentration bound that was developed and utilized by the recent work of Simchowitz et al \cite{lwom}. Below we state our bound in the more practical setup $m\leq n$ to avoid redundant notation. In general, our bound scales with $\min\{m,n\}$. %
\begin{theorem} \label{decomposed bound}Define $\gamma=\frac{\|\Ginf\|\Phi(\A)^2\|\Cb\A^{\TK-1}\|^2}{1-\rho(\A)^{2\TK}}$. $\Ub_t^*\Eb_t$ obeys 
\[
\|\Ub_t^*\Eb_t\|\leq  c_0\sigma_u\sqrt{\tau \max\{\TK p,m\}  N_t \gamma},
\] with probability at least $1-\exp(-100\max\{\TK p,m\})-2\exp(-c\tau N_t(1-\rho(\A)^\TK)+3m)$ for $\tau\geq 1$.
\end{theorem}
\begin{proof} Given matrices $\Ub_t,\Eb_t$, define the filtrations $\Fc_i=\sigma(\{\ub_j,\w_j\}_{j=1}^{t+i\TK})$ for $1\leq i\leq N_t$. According to this definition $\ubb_{t+i\TK}$ is independent of $\Fc_{i-1}$ and $\ubb_{t+i\TK}\in \Fc_i$. The reason is earliest input vector contained by $\ubb_{t+i\TK}$ has index $t+1+(i-1)\TK$ which is larger than $t+(i-1)\TK$. Additionally, observe that $\eb_{t+i\TK}\in \Fc_{i-1}$ as $\eb_{t+i\TK}$ is a deterministic function of $\x_{t+1+(i-1)\TK}$ which is a function of $\{\ub_j,\w_j\}_{j=1}^{t+(i-1)\TK}$. 

We would like to use the fact that, for each $i$, $\eb_{t+i\TK}$ and $\ubb_{t+i\TK}$ are independent. Let $\X_t=[\x_{t+1}~\dots~\x_{t+1+(N_t-1)\TK}]^*$ so that $\Eb_t=\X_t(\Cb\A^{\TK-1})^*$. In light of Lemma \ref{cover bound}, we will use a covering bound on the matrix
\[
\Ub_t^*\Eb_t=\Ub_t^*\X_t(\Cb\A^{\TK-1})^*.
\]
Let $\Cc_1$ be a $1/4$ $\ell_2$-cover of the unit sphere $\Sc^{\TK p-1}$ and $\Cc_2$ be a $1/4$ $\ell_2$-cover of the unit sphere in the row space of $\Cb$. There exists such covers satisfying $\log |\Cc_1|\leq 3\TK p$ and $\log |\Cc_2|\leq 3\min\{m,n\}\leq 3m$. Pick vectors $\ab,\bb$ from $\Cc_1,\Cc_2$ respectively. Let $W_i=\ab^*\ubb_{t+i\TK}$ and $Z_i=\bb^*\e_{t+i\TK}$. Observe that
\[
\sum_{i=1}^{N_t} W_iZ_i=\ab^*(\Ub_t^*\Eb_t)\bb.
\]
We next show that $\sum_{i=1}^{N_t}W_iZ_i$ is small with high probability. Applying Lemma \ref{thm e bound}, we find that, for $\tau\geq 2$, with probability at least $1-2\exp(-c\tau N_t(1-\rho(\A)^\TK))$,
\begin{align}
\tn{\Eb_t\bb}^2=\sum_{i=1}^{N_t}Z_i^2\leq \tau N_t\gamma,\label{prob b bound}
\end{align}
where our definition of $\gamma$ accounts for the $\|\Ginf\|$ factor. We will use this bound to ensure Lemma \ref{sig sub} is applicable with high probability. Since $\ubb_{t+i\TK}$ has $\Nn(0,\sigma_u^2)$ entries, applying Lemma \ref{sig sub}, we obtain
\[
\Pro(\{\sum_{i=1}^{N_t}W_iZ_i\geq t\}\bigcap \{\sum_{i=1}^{N_t}Z_i^2\leq \tau N_t\gamma\})\leq \exp(-\frac{t^2}{c\tau \sigma_u^2N_t\gamma}).
\] 
for some absolute constant $c>0$. Picking $t=11\sigma_u\sqrt{c\tau\max\{\TK p,m\} N_t \gamma}$, we find
\[
\Pro(\{\sum_{i=1}^{N_t}W_iZ_i\geq t\}\bigcap \{\sum_{i=1}^{N_t}Z_i^2\leq cN_t\gamma\})\leq \exp(-120\max\{\TK p,m\}).
\]
Defining variables $W_i(\ab)$ for each $\ab\in\Cc_1$, and events $E(\ab)=\{\sum_{i=1}^{N_t}W_i(\ab)Z_i\geq t\}$, applying a union bound, we obtain,
\[
\Pro(\{\bigcup_{\ab\in\Cc_1}E(\ab)\}\bigcap \{\sum_{i=1}^{N_t}Z_i^2\leq cN_t\gamma\})\leq \exp(-110\max\{\TK p,m\}).
\]
Combining this bound with \eqref{prob b bound}, we find that, for a \em{fixed} $\bb$ and \em{for all} $\ab$, with probability at least $1-\exp(-110\max\{\TK p,m\})-2\exp(-c\tau N_t(1-\rho(\A)^\TK))$, we have 
\begin{align}
\ab^*\Ub_t^*\Eb_t\bb=\sum_{i=1}^{N_t}W_iZ_i\leq c_0\sigma_u\sqrt{\tau\max\{\TK p,m\} N_t \gamma},\label{single bound}
\end{align}
for some $c_0>0$. Applying a union bound over all $\bb\in\Cc_2$, with probability at least $1-\exp(-100\max\{\TK p,m\})-2\exp(-c\tau N_t(1-\rho(\A)^\TK)+3m)$, we find that \eqref{single bound} holds for all $\ab,\bb$. Overall, we found that for all $\ab,\bb$ pairs in the $1/4$ covers, $\ab^*(\Ub_t^*\Eb_t)\bb\leq \kappa=c_0\sigma_u\sqrt{\tau\max\{\TK p,m\} N_t\gamma}$. Applying Lemma \ref{cover bound}, this implies $\|\Ub_t^*\Eb_t\|\leq 2\kappa$.
\end{proof}
The following corollary simplifies the result when $N\geq \TK$ which is the interesting regime for our purposes.
\begin{corollary}\label{simplified} Assume $N\geq T$. With probability at least $1-\exp(-100\TK p)-2\exp(-100m)$, we have $\|\Ub_t^*\Eb_t\|\leq c'\sigma_u\sqrt{\max\{N,\frac{m\TK}{1-\rho(\A)^{\TK}}\} \max\{p,m/\TK\}  \gamma}$ for some constant $c'>0$.
\end{corollary}
\begin{proof}$N\geq \TK$ implies $N_t\geq \lfloor N/\TK\rfloor\geq N/(2\TK)$. In Theorem \ref{decomposed bound}, pick $\tau= \max\{1,c_1\frac{m\TK}{N(1-\rho(\A)^{\TK})}\}$ for $c_1= 206/c$. The choice of $\tau$ guarantees the probability exponent $c\tau N_t(1-\rho(\A)^{\TK})-3m\geq 100m$. To conclude, observe that $c_0\sigma_u\sqrt{\tau \max\{\TK p,m\}  N_t \gamma}\leq c'\sigma_u\sqrt{\max\{1,\frac{m\TK}{N(1-\rho(\A)^{\TK})}\} \max\{p,m/\TK\}N\gamma}$ for an absolute constant $c'>0$.
\end{proof}
For completeness, we restate the subgaussian Martingale concentration lemma of Simchowitz et al. which is Lemma $4.2$ of \cite{lwom}.
\begin{lemma} \label{sig sub}Let $\{\Fc_t\}_{t\geq 1}$ be a filtration, $\{Z_t,W_t\}_{t\geq 1}$ be real valued processes adapted to $\Fc_t,\Fc_{t+1}$ respectively (i.e.~$Z_t\in\Fc_t,W_t\in\Fc_{t+1}$). Suppose $W_t\bgl \Fc_t$ is a $\sigma^2$-sub-gaussian random variable with mean zero. Then
\[
\Pro(\{\sum_{t=1}^\TK Z_tW_t\geq \alpha\}\bigcap \{\sum_{t=1}^TZ_t^2\leq \beta\})\leq \exp(-\frac{\alpha^2}{2\sigma^2\beta})
\]
\end{lemma}
This lemma implies that $\sum_{t=1}^T Z_tW_t$ can essentially be treated as an inner product between a deterministic sequence $Z_t$ and an i.i.d.~subgaussian sequence $W_t$.

The following lemma is a slight modification of the standard covering arguments.
\begin{lemma} [Covering bound]\label{cover bound} Given matrices $\A\in\R^{n_1\times N},\B\in\R^{N\times n_2}$, let $\M=\A\B$. Let $\Cc_1$ be a $1/4$-cover of the unit sphere $\Sc^{n_1-1}$ and $\Cc_2$ be a $1/4$-cover of the unit sphere in the row space of $\B$ (which is at most $\min\{N,n_2\}$ dimensional). Suppose for all $\ab\in\Cc_1,\bb\in\Cc_2$, we have that $\ab^*\M\bb\leq \gamma$. Then, $\|\M\|\leq 2\gamma$.
\end{lemma}
\begin{proof} Pick unit length vectors $\x,\y$ achieving $\x^*\M\y=\|\M\|$. Let $S$ be the row space of $\B$. Observe that $\y\in S$. Otherwise, its normalized projection on $S$, $\Pc_S(\y)/\tn{\Pc_S(\y)}$ achieves a strictly better inner product with $\x^*\M$. Pick $1/4$ close neighbors $\ab,\bb$ of $\x,\y$ from the covers $\Cc_1,\Cc_2$. Then,
\[
\x^*\M\y= \ab^*\M\bb+(\x-\ab)^*\M\bb+\x^*\M(\y-\bb)\leq\gamma+ \x^*\M\y/2,
\]
where we used the maximality of $\x,\y$. This yields $\x^*\M\y\leq 2\gamma$.
\end{proof}

\subsection{Bounding Inner Products with the Unknown State}\label{error mat sec}
In this section, we develop probabilistic upper bounds for the random variable $\Eb_t\ab$ where $\ab$ is a fixed vector and $\Eb_t$ is as defined in \eqref{decoupled}.
\begin{lemma} \label{thm e bound}Let $\Eb_t\in\R^{N_t\times m}$ be the matrix composed of the rows $\eb_{t+i\TK}=\Cb\A^{\TK-1}\x_{t+1+i\TK}$. Define 
\[
\gamma=\frac{\Phi(\A)^2\|\Cb\A^{\TK-1}\|^2}{1-\rho(\A)^{2\TK}}.
\]
Given a unit length vector $\ab\in\R^m$, for all $\tau\geq 2$ and for some absolute constant $c>0$, we have that
\[
\Pro(\tn{\Eb_t\ab}^2\geq \tau N_t\|\Ginf\|\gamma)\leq  2\exp(-c\tau N_t(1-\rho(\A)^{\TK})).
\]
\end{lemma}
\begin{proof}
Let $\db_t=\x_{t}-\A^\TK\x_{t-\TK}$. By construction (i.e.~due to the state-space recursion \eqref{main rel}), $\db_t$ is independent of $\x_{t-\TK}$. We can write $\x_{t+i\TK}$ as
\begin{align}
\x_{t+i\TK}=\sum_{j=1}^i\A^{(i-j)\TK}\db_{t+j\TK}+\A^{i\TK}\x_t.\label{expand x}
\end{align}
We wish to understand the properties of the random variable $\tn{\Eb_t\ab}^2$ which is same as,
\[
s_{\ab}=\sum_{i=1}^{N_t}(\ab^*\e_{t+i\TK})^2=\sum_{i=0}^{N_t-1}((\ab^*\Cb\A^{\TK-1})\x_{t+1+i\TK})^2.
\]
Denote $\bar{\ab}=(\Cb\A^{\TK-1})^*\ab$, $\ab_j=(\A^{j\TK})^*\bar{\ab}$, $\g_0=\x_{t+1}$, and $\g_i=\db_{t+1+i\TK}$ for $N_t-1\geq i\geq 1$, all of which are $n$ dimensional vectors. Using these change of variables and applying the expansion \eqref{expand x}, the $i$th component of the sum $s_{\ab}$ is given by
\begin{align}
s_{\ab,i}=(\bar{\ab}^*\x_{t+1+i\TK})^2=(\bar{\ab}^*\sum_{j=0}^i\A^{(i-j)\TK}\g_j)^2=(\sum_{j=0}^i\ab_{i-j}^*\g_j)^2=\sum_{0\leq j,k\leq i}\ab_{i-j}^*\g_j\ab_{i-k}^*\g_k.\label{sabi eq}
\end{align}
Observe that, summing over all $s_{\ab,i}$ for $0\leq i\leq N_t-1$, the multiplicative coefficient of the $\g_j\g_k^*$ pair is given by the matrix,
\begin{align}
\M_{j,k}=\begin{cases}\underset{N_t>i\geq \max\{j,k\}}{\sum} \ab_{i-j}\ab_{i-k}^*~~~&\text{if}~~~j\neq k,\\\underset{N_t>i\geq j}{\sum} \ab_{i-j}\ab_{i-j}^*=\sum_{i=0}^{N_t-1-j} \ab_{i}\ab_{i}^*~~~&\text{if}~~~j= k\end{cases}\label{mjk eq}
\end{align}
Next, we show that these $\M_{j,k}$ submatrices have bounded spectral, Frobenius and nuclear norms (nuclear norm is the sum of the singular values of a matrix). This follows by writing each submatrix as a sum of rank $1$ matrices and using the fact that spectral radius of $\A$ is strictly bounded from above by $1$.
\begin{align*}
\|\M_{j,k}\|\leq \|\M_{j,k}\|_F\leq  \|\M_{j,k}\|_\star&\leq \sum_{i\geq \max\{j,k\}} \|\ab_{i-j}\ab_{i-k}^*\|_\star\\
&= \sum_{i\geq \max\{j,k\}} \|\ab_{i-j}\ab_{i-k}^*\|\\
&\leq \sum_{i\geq \max\{j,k\}} \|(\A^{(i-j)\TK})^*\bar{\ab}\bar{\ab}^*\A^{(i-k)\TK}\|\\
&\leq \sum_{i\geq \max\{j,k\}} \tn{\bar{\ab}}^2\|\A^{(i-j)\TK}\|\|\A^{(i-k)\TK}\|\\
&\leq \sum_{i=0}^{\infty} \tn{\bar{\ab}}^2\rho(\A)^{|j-k|\TK}\rho(\A)^{2i\TK}\Phi(\A)^2\\
&\leq  \frac{\Phi(\A)^2\tn{\bar{\ab}}^2}{1-\rho(\A)^{2\TK}}\rho(\A)^{|j-k|\TK}.
\end{align*}
To further simplify, observe that $\tn{\bar{\ab}}^2\leq \|\Cb\A^{\TK-1}\|^2$ as $\tn{\ab}=1$. Setting
\[
\gamma=\frac{\Phi(\A)^2\|\Cb\A^{\TK-1}\|^2}{1-\rho(\A)^{2\TK}},
\]
we have 
\begin{align}
\|\M_{j,k}\|,\|\M_{j,k}\|_F,\|\M_{j,k}\|_\star\leq \gamma \rho(\A)^{|j-k|\TK}.\label{mjk bound}
\end{align} Based on the submatrices $\M_{j,k}$, create the $N_tn\times N_tn$ matrix $\M$. Now we define the vector $\bar{\g}=[\g_0^*~\g_1^*~\dots~\g_{N_t-1}^*]^*$. Observe that, following \eqref{sabi eq} and \eqref{mjk eq}, by construction,
\begin{align}
s_{\ab}=\bar{\g}^*\M\bar{\g}=\sum_{0\leq j,k<N_t}\g_j^*\M_{j,k}\g_k.\label{quad form}
\end{align}
This puts $s_{\ab}$ in a form for which Hanson-Wright Theorem is applicable \cite{adamczak2015note,rudelson2013hanson}. To apply Hanson-Wright Theorem, let us first bound the expectation of $s_{\ab}$. Since $\{\g_i\}_{i=0}^{N_t-1}$'s are truncations of the state vector, we have that $\bSi(\g_i)\preceq\bSi(\x_{t+1+i\TK})\preceq \Ginf$. Write $\g_i=\bSi(\g_i)^{1/2}\h_i$ for some $\h_i\sim\Nn(0,\Iden_n)$. Using independence of $\h_i, \h_j$ for $i\neq j$ and $\bSi(\g_i)\preceq\Ginf$, we have that
\begin{align}
\E[s_{\ab}]&=\sum_{i=0}^{N_t-1}\E[\g_i^*\M_{i,i}\g_i]=\sum_{i=0}^{N_t-1}\E[\h_i^*\bSi(\g_i)^{1/2}\M_{i,i}\bSi(\g_i)^{1/2}\h_i]\\
&=\sum_{i=0}^{N_t-1}\tr{\bSi(\g_i)^{1/2}\M_{i,i}\bSi(\g_i)^{1/2}}\label{trace relation}\\
&\leq \sum_{i=0}^{N_t-1}\|\bSi(\g_i)\|\tr{\M_{i,i}}\leq \sum_{i=0}^{N_t-1}\|\Ginf\|\tr{\M_{i,i}}\\
&\leq  N_t\|\Ginf\|\gamma.\label{mean sa eq}
\end{align}
In \eqref{trace relation}, we utilized the fact that for positive semidefinite matrices trace is equal to the nuclear norm and then we used the fact that nuclear norm of the product obeys $\|\X\Y\|_\star\leq \|\X\|_\star\|\Y\|$ \cite{hogben2006handbook}. Finally, we upper bounded $\|\bSi(\g_i)\|$ by using the relation $\bSi(\g_i)\preceq \Ginf$. Bounded $\|\bSi(\g_i)\|$ also implies that the Gaussian vector $\g_i$ obeys the ``concentration property'' (Definition $2.1$ of \cite{adamczak2015note}) with $K=\order{\sqrt{\|\Ginf\|}}$ as Lipschitz functions of Gaussians concentrate. Recalling \eqref{quad form}, the Hanson-Wright Theorem of \cite{adamczak2015note} states that
\[
\Pro(s_{\ab}\geq \E[s_{\ab}]+t)\leq 2\exp(-c\min\{\frac{t^2}{\|\Ginf\|^2\|\M\|_F^2},\frac{t}{\|\Ginf\|\|\M\|}\}).
\]
To proceed, we upper bound $\|\M\|_F$ and $\|\M\|$. First, recall again that $\|\M_{i,j}\|_F\leq \gamma \rho(\A)^{|i-j|\TK}$. Adding these over all $i,j$ pairs, using \eqref{mjk bound} and the fact that there are at most $2N_t$ pairs with fixed difference $|i-j|=\tau$, we obtain
\[
\|\M\|_F^2=\sum_{i,j} \|\M_{i,j}\|_F^2\leq \sum_{0\leq i,j\leq N_t-1}\gamma^2\rho(\A)^{2|i-j|\TK}\leq 2N_t\gamma^2\sum_{\tau=0}^{N_t-1} \rho(\A)^{2\tau\TK}\leq \frac{2\gamma^2N_t}{1-\rho(\A)^{2\TK}}.
\]
To assess the spectral norm, we decompose $\M$ into $2N_t-1$ block permutation matrices $\{\M^{(i)}\}_{i=-N_t+1}^{N_t-1}$. $\M^{(0)}$ is the main diagonal of $\M$, and $\M^{(i)}$ is the $i$th off-diagonal that contains only the submatrices $\M_{j,k}$ with fixed difference $j-k=i$. By construction $\|\M^{(i)}\|\leq \gamma \rho(\A)^{|i|\TK}$ as each nonzero submatrix satisfies the same spectral norm bound. Hence using \eqref{mjk bound},
\[
\|\M\|\leq \sum_{i=-N_t+1}^{N_t-1}\|\M^{(i)}\|\leq \gamma(\frac{2}{1-\rho(\A)^{\TK}}-1)\leq \frac{2\gamma}{1-\rho(\A)^{\TK}}.
\]
With these, setting $t=\tau{N_t\|\Ginf\|\gamma}$ and using \eqref{mean sa eq} and bounds on $\tf{\M},\|\M\|$, for $\tau\geq 1$ and using $K=\order{\sqrt{\|\Ginf\|}}$, and applying Theorem $2.3$ of \cite{adamczak2015note}, we find the concentration bound
\begin{align}
\Pro(s_{\ab}\geq (\tau+1)N_t\|\Ginf\|\gamma)&\leq  2\exp(-2c\tau\min\{\frac{(N_t\|\Ginf\|\gamma)^2}{\|\Ginf\|^2\frac{2\gamma^2N_t}{1-\rho(\A)^{2\TK}}},\frac{N_t\|\Ginf\|\gamma}{\|\Ginf\|\frac{2\gamma}{1-\rho(\A)^{\TK}}}\})\\
&\leq 2\exp(-c\tau\min\{{N_t(1-\rho(\A)^{2\TK})},{N_t(1-\rho(\A)^{\TK})}\})\\
&= 2\exp(-c\tau N_t(1-\rho(\A)^{\TK})),
\end{align}
which is the desired result after $1+\tau\leftrightarrow \tau$ substitution and using the initial assumption of $\tau\geq 2$.
\end{proof}

\section{Proof of Theorem \ref{circ thm}}\label{sec circ thm}
This proof is a slight modification of the proof of Theorem $4.1$ of Krahmer et al. \cite{krahmer2014suprema} and we will directly borrow their notation and estimates. First, we restate their Theorem $3.1$. 
\begin{theorem} \label{thm 31}Let $\Ac$ be a set of matrices and let $\xi$ be a random vector whose entries $\xi_j$ are standard normal. Let $d_F,d_{2\rightarrow 2}$ be the Frobenius and spectral norm distance metrics respectively. Set
\begin{align}
&E=\gamma_2(\Ac,\|\cdot\|)(\gamma_2(\Ac,\|\cdot\|)+d_F(\Ac))+d_F(\Ac)d_{2\rightarrow 2}(\Ac),\\
&V=d_{2\rightarrow 2}(\Ac)(\gamma_2(\Ac,\|\cdot\|)+d_F(\Ac)),~\text{and}~U=d^2_{2\rightarrow 2}(\Ac).
\end{align}
Then, for some absolute constants $c_1,c_2>0$ and for all $t>0$,
\[
\Pro(\sup_{\A\in\Ac}|\tn{\A\xi}^2-\E\tn{\A\xi}^2|\geq c_1 E+t)\leq\exp(-c_2\min\{\frac{t^2}{V^2},\frac{t}{U}\}).
\]
\end{theorem}
Theorem \ref{sec circ thm} is a variation of Theorem $4.1$ of \cite{krahmer2014suprema}. In light of Theorem \ref{thm 31}, we simply need to adapt the estimates developed during the proof of Theorem $4.1$ of \cite{krahmer2014suprema} for our purposes. We are interested in a fixed submatrix of size $m\times s$ compared to all $s$-column submatrices for fixed $m$-rows. This makes our set $\Ac$ a subset of their set and also makes their estimates an upper bound on our estimates. Following arguments of \cite{krahmer2014suprema}, for some constant $c_3>0$, we have
\[
d_F(\Ac)=1,~d_{2\rightarrow 2}(\Ac)\leq \sqrt{s/m},~\gamma_2(\Ac,\|\cdot\|)\leq c_3{\sqrt{s/m}\log(2s)\log(2d)}.
\]
To proceed, we will apply Theorem \ref{thm 31}. This will be done in two scenarios depending on whether isometry constant obeys $\delta\leq 1$ or not. Recall that $m_0=c_0s\log^2(2s)\log^2(2d)$. Below, we pick $c_0$ sufficiently large to compensate for $c_1,c_2,c_3$.

\noindent {\bf{$m\geq m_0$ case:}} We have that $\gamma_2(\Ac,\|\cdot\|)\leq c_3{\sqrt{s/m}\log (2s)\log (2d)}\leq 1$ so that $E\leq \sqrt{s/m}+2\gamma_2(\Ac,\|\cdot\|)\leq 3c_3\sqrt{s/m}\log(2s)\log(2d)$. Similarly, $V\leq 2\sqrt{s/m}$ and $U\leq s/m$. In this case, picking large $c_0$, observe that $c_1 E\leq \sqrt{m_0/(4m)}$. With this, we can pick $t=\sqrt{m_0/(4m)}$ to guarantee $c_1E+t\leq \sqrt{m_0/m}$. We have that $t^2/V^2\geq m_0/(16s),~t/U\geq t^2/U\geq m_0/(4s)$. Picking $c_0\geq 16/c_2$, we conclude with the desired probability $\exp(-\log^2(2d)\log^2(2s))$.

\noindent {\bf{$m< m_0$ case:}} In this case, we have
\[
\gamma_2(\Ac,\|\cdot\|)+d_F(\Ac)\leq c'{\sqrt{s/m}\log(2s)\log(2d)},
\]
where $c'=c_3+\sqrt{c_0}$. Hence, we find
\[
E\leq c'c_3(s/m)\log(2s)^2\log(2d)^2+\sqrt{s/m}.
\]
Observe that, we can ensure i) $c_1\sqrt{s/m}\leq \sqrt{m_0/m}/4\leq m_0/(4m)$ and ii) $c_1c'c_3(s/m)\log(2s)^2\log(2d)^2\leq m_0/(4m)$ for sufficiently large constant $c_0$. The latter one follows from the fact that $c'$ grows proportional to $\sqrt{c_0}$ whereas $m_0$ grows proportional to $c_0$. With this, we can pick $t=m_0/(2m)$ which guarantees $c_1E+t\leq \frac{m_0}{m}$.

To find the probability, we again pick $c_0$ to be sufficiently large to guarantee that i) $c_2t/U\geq c_2(m_0/(2m))/(s/m)\geq \log^2(2s)\log^2(2d)$ and ii)
\begin{align*}
t^2/V^2&\geq \frac{(m_0/m)^2}{4(s/m) (c'{\sqrt{s/m}\log(2s)\log(2d)})^2}\\
&= \frac{c_0^2s^2\log^4(2s)\log^4(2d)}{4s^2(c')^2\log^2(2s)\log^2(2d)}\\
&=\frac{c_0^2\log^2(2s)\log^2(2d)}{4(c_3+\sqrt{c_0})^2}\geq \log^2(2s)\log^2(2d)/c_2,
\end{align*}
which concludes the proof by yielding $\exp(-\log^2(2d)\log^2(2s))$ probability of success.

\end{document}